\documentclass[letterpaper]{article} 
\usepackage[preprint]{aaai2027}  
\usepackage[hyphens]{url}  
\usepackage{graphicx} 
\usepackage{natbib}  
\usepackage{caption} 
\usepackage{algorithm}
\usepackage{algorithmic}

\usepackage{booktabs}
\usepackage{microtype}
\usepackage{amsmath}
\usepackage{amssymb}
\usepackage{amsthm}
\usepackage{multirow}
\usepackage{array}
\usepackage{adjustbox}
\usepackage{dsfont}

\newtheorem{problem}{Problem}

\theoremstyle{plain}
\newtheorem{theorem}{Theorem} 

\newtheorem{lemma}[theorem]{Lemma}

\theoremstyle{remark}

\newcommand{\ours}{{IB-Forecast}}

\title{Information Bottleneck Learning for Faithful Time Series Forecasting Explanations}
\author{
    Xu Zheng\textsuperscript{\rm 1},
    Wei Cheng\textsuperscript{\rm 2},
    Zhuomin Chen\textsuperscript{\rm 1},
    Mo Sha\textsuperscript{\rm 1},
    Jingchao Ni\textsuperscript{\rm 3},
    Dongsheng Luo\textsuperscript{\rm 4}\corresponding
}
\affiliations{
    \textsuperscript{\rm 1}Florida International University, Miami, FL, USA\\
    \textsuperscript{\rm 2}NEC Laboratories America, Princeton, NJ, USA\\
    \textsuperscript{\rm 3}University of Houston, Houston, TX, USA\\
    \textsuperscript{\rm 4}Singapore Management University, Singapore\\
    xzhen019@fiu.edu, weicheng@nec-labs.com, zchen051@fiu.edu,\\
    msha@fiu.edu, jni7@uh.edu, luodongsheng01@gmail.com
}

\begin{document}

\maketitle

\begin{abstract}
As forecasts increasingly drive decisions in fields such as energy, transportation, and healthcare, understanding the historical data behind these predictions has become as crucial as the predictions themselves. Although existing interpretable-by-design forecasters reveal their internal structures, they offer no guarantee that these structures faithfully reflect the underlying evidence driving the predictions. In contrast, while faithfulness-oriented methods explicitly verify model behavior, they are almost exclusively designed for post-hoc classification tasks. To bridge this gap, we propose {\ours}, an inherently interpretable multivariate time-series forecasting framework. It decomposes forecasting into a learned periodic component and a residual component computed with explainable masks over input tokens. With a budget-constrained information bottleneck, end-to-end optimization enables users to directly control explanation sparsity. With a rigorous faithfulness evaluation protocol, extensive experiments demonstrate that {\ours} matches the forecasting error of leading black-box models while providing faithful explanations at no additional inference cost. Furthermore, under a matched sparsity budget, these native explanations consistently surpass gradient-based, occlusion-based, and optimization-based baselines across all evaluated datasets. Ultimately, whereas the native explanations of existing interpretable forecasters exhibit poor faithfulness, {\ours} guarantees high explanation fidelity, requiring only 14--20\% of the observations to deliver low-error predictions.
\end{abstract}

\section{Introduction}
\label{sec:intro}
Long-horizon multivariate forecasting increasingly supports decision-making in energy~\cite{informer2021}, traffic~\cite{li2018diffusion}, weather~\cite{rasp2020weatherbench}, and healthcare~\cite{johnson2016mimic}. Recent advances have produced a rapid succession of increasingly accurate forecasting architectures~\cite{informer2021,patchtst2023,itransformer2024,cyclenet2024,tqnet2025}. When forecasts drive consequential decisions, understanding which historical inputs influenced a prediction is crucial~\cite{rudin2019stop}. However, the most accurate forecasters remain notoriously opaque and lack inherent interpretability~\cite{lim2021time}. 

To unravel the black boxes, two mainstream kinds of methods are introduced. One line of work addresses this problem through intrinsically interpretable forecasting architectures, including the basis decomposition of N-BEATS and N-HiTS~\citep{nbeats2020,nhits2023}, the variable-selection networks and interpretable attention mechanisms of the Temporal Fusion Transformer~\citep{tft2021}, prototype-based models~\citep{prosenet2019,protots2025}, and concept-bottleneck approaches~\citep{cbmts2024}. These methods are typically evaluated primarily on forecasting error, while the interpretability provided by their internal structures is rarely assessed systematically. 
Another line of work develops faithfulness-oriented attribution methods that explain a trained predictor after the fact.  Recent information-bottleneck approaches~\citep{timex2023,contralsp2024,jang2025timing} isolate minimal input subsets that sustain predictive performance, frequently establishing robust evaluation protocols and benchmarks in the process. However, these approaches focuses predominantly on classification and relies on post-hoc explanations for frozen predictors.This decouples the explanation from the actual forecasting computation, frequently requiring the model to evaluate perturbed inputs that fall outside its training distribution~\cite{timexpp2024}. 

Compared to standard classification, multi-step forecasting requires modeling persistent levels and scales, recurring temporal structures, and input-dependent biases across both variables and future time horizons, making a direct transfer of classification-oriented attribution insufficient. Moreover, a forecast is often determined jointly by regular patterns learned across the training data, such as level, scale, and periodicity, and by instance-specific observations in the current history window~\cite{timesfm,nbeats2020}. 

To bridge this gap, we introduce {\ours}, a self-interpretable multivariate forecasting framework in which the forecast depends on the history window through a disclosed structural context, namely each channel's level and scale together with a learned seasonal profile, and the small set of history tokens selected by a sparse binary mask. Conditioned on this context, the forecast depends only on the emitted mask and the tokens it selects.
The explanation is generated within the forward pass and forms part of the forecasting computation itself. 
To achieve this property, {\ours} decomposes each prediction into two readable components by following PatchDecomp~\cite{patchdecomp2026} to map the input into tokens. First, a learned periodic profile captures recurring temporal patterns, such as daily or weekly variation. Second, a gated deviation readout describes which departures from those recurring patterns are inferred from the current history window. 
Beyond the disclosed context, all input dependence passes through a budgeted binary mask over patch and channel deviation tokens. Closing a gate prevents its deviation from influencing the readout when the structural context is held fixed.

From the information theoretic perspective, an explanation should preserve the information needed to predict the future while discarding input details that are unnecessary for forecasting~\cite{infomin}.  With this principle, Information Bottleneck(IB) ~\citep{tishby2000}, we learn the mask by balancing two objectives, retaining sufficient information about the forecasting target and limiting the amount of historical information that can pass through the gated deviation pathway. This balance improves the stability and generalization of both the forecast and its explanation by restricting redundant or unstable details. 
We further introduce an explicit information budget that allows users to control how much of the historical window the model may consult.  Crucially, {\ours} guarantees faithfulness by architecturally ensuring that only the tokens selected by the information bottleneck can affect the deviation readout. 
With a relaxed information budget, {\ours} approaches its dense forecasting configuration and can use the complete history window.  As the budget becomes tighter, the bottleneck forces the model to prioritize the observations that carry the most forecasting information. This produces a spectrum of models ranging from dense forecasting to highly selective forecasting within a single framework.  Across multivariate forecasting benchmarks, {\ours} retains competitive predictive accuracy under tight budgets while consulting only a small fraction of the historical input. In summary, we make three main contributions:
\begin{itemize} 
    \item We formulate self-interpretable multivariate forecasting in Problem~\ref{prob:sif}, which asks a single forward pass to jointly produce a forecast and a compact explanation under a predefined budget and a disclosed structural context.
    \item We propose {\ours}, which decomposes each forecast into a learned periodic profile and a gated deviation readout. The gate serves as an information bottleneck, retaining forecasting-relevant information while limiting unnecessary dependence on the historical input.
    \item Across multivariate benchmarks, {\ours} matches strong dense forecasters in accuracy, yields more faithful masks than intrinsic and post-hoc baselines, and outperforms test-time optimization without inference-time optimization. The results also show that accurate forecasts often require only a small fraction of historical input.
\end{itemize}

\section{Related Work}
\label{sec:related}

\noindent\textbf{Interpretable-by-design forecasting.}
Recent forecasting models achieve interpretability through various structural priors. N-BEATS and NHiTS \citep{nbeats2020,nhits2023} decompose forecasts into interpretable bases, while DLinear \citep{dlinear2023} provides transparency in the classical sense through a readable linear map. Other approaches expose specific architectural components to users, such as variable-selection weights and attention in TFT \citep{tft2021}, per-patch contributions in PatchDecomp \citep{patchdecomp2026}, prototype similarities in ProSeNet and ProtoTS \citep{prosenet2019,protots2025}, and named concepts in concept-bottleneck transformers \citep{cbmts2024}. However, these methods failed to evaluate exposed structures as faithful explanations. 

\noindent\textbf{Post-hoc time-series explanation.}
Gradient-based techniques \citep{saliency2014,ig2017,jang2025timing} and perturbation methods dominate post-hoc explanation. A prominent recent line of work focuses on learning per-instance masks at test time, including Dynamask \citep{dynamask2021}, ExtremalMask \citep{extremalmask2023}, and ContraLSP \citep{contralsp2024}. Information-bottleneck explainers such as TimeX \citep{timex2023} and TimeX++ \citep{timexpp2024} bring information-theoretic rigor and ground-truth benchmarks to this space.  Due to designed for classification after training, they can disagree with the computation \citep{rudin2019}, which makes robust evaluation an active research area itself \citep{eraser2020,zheng2025ffidelity}. 

\noindent\textbf{Information bottlenecks and stochastic gates.}
The information bottleneck principle \citep{tishby2000} was ported to interpretable graph learning by GSAT \citep{gsat2022}, whose stochastic-attention argument, in which injected gate noise penalizes retaining merely correlated inputs, is the blueprint for our gate objective. We use hard binary-Concrete gates \citep{concrete2017,l0reg2018} with straight-through estimation \citep{ste2013}.

\begin{figure*}[t]
\centering
\adjustbox{trim=0cm 0.5cm 0cm 0cm}{
\includegraphics[width=\textwidth]{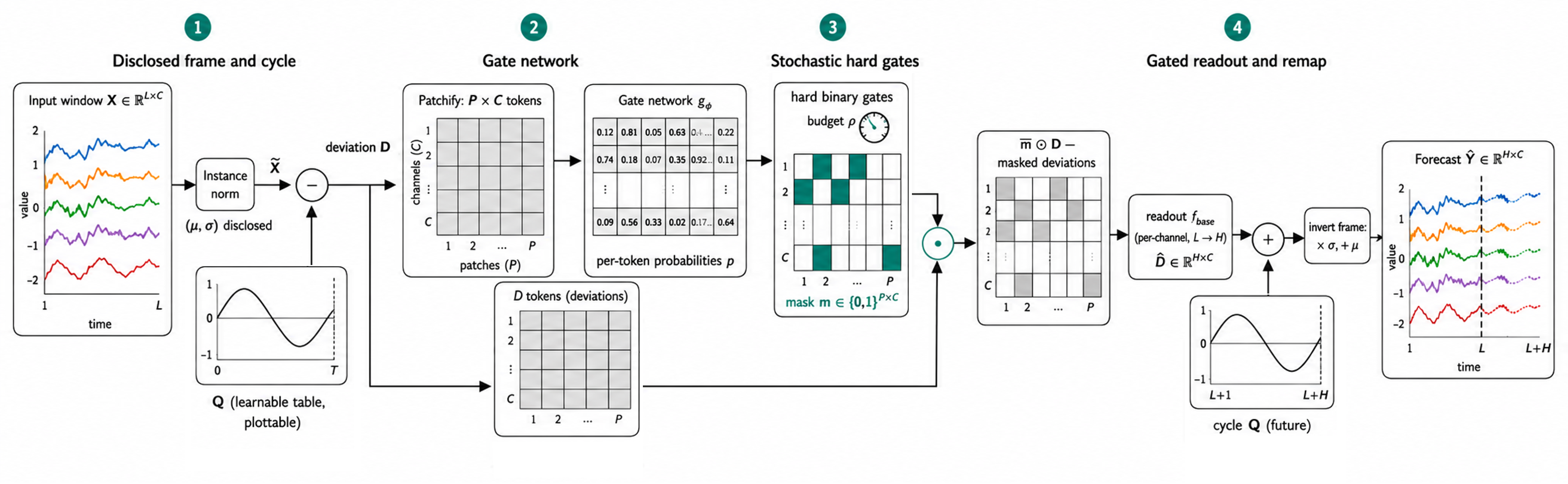}
}
\caption{Overview of {\ours}. Stage 1 normalizes the input $X$ and removes the learned periodic profile $Q$ to obtain the deviation $D$. Stage 2 partitions $D$ into tokens and scores each with a small transformer gate network.
Stage 3 samples a hard binary gate for each token to form the mask $m$. 
Stage 4 computes the forecast from the masked deviations $\bar m \odot D$ and remaps the output back to the original space.
The mask $m$, profile $Q$, and normalization information constitute the explanation $E$.} 
\label{fig:arch}
\end{figure*}
    
\section{Preliminaries}
\label{sec:setup}

\noindent \textbf{Time Series Decomposition.}
A time series is classically decomposed into three components: a slowly varying \emph{trend}, a recurring \emph{seasonal} pattern, and a \emph{residual} that remains after the first two are removed \citep{stl1990}. For each channel, the decomposition is additive:
\begin{equation}
X[t,c] = T[t,c] + S[t,c] + R[t,c],
\label{eq:classical}
\end{equation}
where $T$ is the trend, $S$ is the seasonal component with period $W$, such that
$S[t,c]=S[t \bmod W,c]$, and $R$ is the residual. By separating stable structures like seasonality from window-specific residuals and local trends, this decomposition isolates what truly drives the forecast. Instead of repeatedly explaining predictable seasonal baselines, the focus is solely on how unique local variations adapt the model to the current context.

We formalize this distinction through a structural context $\mathcal{C}=c(X)$ determined from the input. It records the recurring seasonal profile together with a per-channel frame, consisting of a level $\mu$ and scale $\sigma$. 
We assume that, \emph{once $\mathcal{C}$ is fixed, this structural frame spans both the look-back and forecast horizons, extending the seasonal profile periodically}. Consequently, \emph{given $\mathcal{C}$, $X$ and its remainder $\Delta$ mutually determine one another, as do the future $Y$ and its remainder $\Delta_Y$}.
Motivated by this separation, we define an explanation as compact, window-specific evidence interpreted within the structural context. In terms of Eq.~\eqref{eq:classical}, the explanation concerns only the residual $R$ and the local variation of the trend $T$.

\noindent \textbf{Forecasting \& Explanation.}
Let $X \in \mathbb{R}^{L \times C}$ denote a look-back window of $C$ channels over $L$ time steps, where $X[t,c]$ is the value of channel $c$ at time step $t$. We use $X$ to denote the input window, from which the context $\mathcal{C}=c(X)$ is constructed. The training set $\mathcal{T} = \left\{ (X_i,Y_i) \mid i\in[N] \right\} $ contains input windows $X_i$ and targets $Y_i\in\mathbb{R}^{H\times C}$, where each target contains the next $H$ observations of all channels. 
A conventional forecaster is a function $f:\mathcal{X}\rightarrow\mathbb{R}^{H\times C}$ learned by minimizing $\mathbb{E}\|f(X)-Y\|_2^2$.

Post-hoc attribution explains a frozen forecaster after training, but it does not guarantee that the forecast was computed through the resulting explanation. We instead require the forecast to be recoverable from the explanation itself. Inspired by the information bottleneck principle \citep{tishby2000}, we formulate the following problem.

\begin{problem}[Self-interpretable forecasting]
\label{prob:sif}
Learn a forecaster that, in a single forward pass, emits a forecast and an explanation, $(\hat{Y},E)=f(X)$, while disclosing a structural context $\mathcal{C}$, such as a seasonal pattern and its associated per-channel frame. The explanation must satisfy two properties:
\begin{enumerate}
    \item \textbf{Compactness}: the attribution $E$ selects only a small, user-controlled fraction of the input.
    \item \textbf{Conditional faithfulness by construction}: conditioned on $\mathcal{C}$, the forecast depends only on $E$; that is,
    $ \hat{Y}=g(E;\mathcal{C}) $ with the readout function $g$.
\end{enumerate}
Thus, once the structural context is fixed, any window-specific information excluded from $E$ cannot affect the forecast.
\end{problem}

The explanation $E$ is generated solely from the observed window. This prevents explanations from peeking at the forecasting horizon.
Conditioned on $\mathcal{C}$, \emph{the construction enforces the Markov chain $ X \rightarrow E \rightarrow \hat{Y}\mid C$}, 
so that $E$ is the only carrier of window-specific predictive information. In terms of Eq.~\eqref{eq:classical}, \emph{$E$ therefore selects evidence about the residual $R$ and local changes in the trend $T$}, while the seasonal component $S$ belongs to $\mathcal{C}$ and is never included in the attribution.
The conditioning applies symmetrically to the forecasting. 
\emph{Predicting $Y$ given $\mathcal{C}$ is therefore equivalent to predicting its remainder $\Delta_Y$}. Accordingly, the explanation should be evaluated on the deseasonalized and normalized future.

\section{Methodology}
\label{sec:method}

In this section, we introduce {\ours}, an interpretable framework for multivariate time-series forecasting whose explanation is generated within the same forward pass and directly determines the forecast. Figure~\ref{fig:arch} illustrates the framework, which implements the decomposition in Eq.~\eqref{eq:classical} under the formulation in Problem~\ref{prob:sif}. 
The framework consists of four stages. {The first} stage isolates the disclosed structural context $\mathcal{C}$ by normalizing each channel to remove its level and scale and then subtracting a learned seasonal profile, leaving only the window-specific deviations. 
{The second} and {third} stages construct a stochastic binary mask that selects a small, user-specified budget of deviation tokens defined over patches and channels. 
{The fourth} stage computes the deviation forecast solely from the selected tokens, after which the seasonal profile, level, and scale are restored. 
A mask applied directly to the raw window would use the budget in property (i) to retransmit the level and recurring rhythm, while setting a closed token to the implausible raw value of zero. 
In contrast, closing a deviation token expresses the in-distribution assumption that the corresponding observation follows the seasonal pattern at the current window level. 
Consequently, the selected deviations form the only window-specific pathway to the output. 
Given the disclosed context $\mathcal{C}$, the mask itself serves as the explanation, and every excluded token is structurally prevented from influencing the forecast. 
The remainder of this section defines the explanation object in Problem~\ref{prob:sif}, presents the four-stage architecture and training objective, and explains how the final explanation is obtained.

\subsubsection{Normalization and seasonal profile.}
The structural context is removed in two steps. First, each channel is normalized at the instance level, removing the window's own per-channel mean and standard deviation $\mu, \sigma \in \mathbb{R}^{1\times C}$ \citep{revin2022}. Second, the recurring seasonal pattern is subtracted \citep{cyclenet2024}. It is stored as a lookup table $Q \in \mathbb{R}^{W\times C}$ of learnable parameters covering one full cycle of length $W$, fixed from the sampling rate (e.g., $W{=}24$ on hourly data with a daily cycle). Entry $Q[w,c]$ is the typical normalized value of channel $c$ at phase $w$ of the cycle, learned jointly with the rest of the model, so a column of $Q$ plots as, for instance, the average daily load curve of that channel. With $\tau$ the window's phase within the cycle, known from its timestamp,
\begin{equation}
\tilde X = \frac{X - \mu}{\sigma}, \qquad
D = \tilde X - Q^{\mathrm{per}}_{[\tau : \tau+L]},
\label{eq:decomp}
\end{equation}
where $Q^{\mathrm{per}}[t,c] = Q[t \bmod W,\, c]$, which lets one cycle cover a look-back and horizon of any length, so when $L > W$ the profile simply tiles across the window. The subtraction leaves in the deviation $D$ only what is specific to this window, the remaining trend movements and residual structure, which the readout forecasts jointly.

\subsubsection{Amortized gate network.}
The mask operates on $(\text{patch}\times\text{channel})$ deviation tokens. We split $D$ into $P=L/L_p$ patches per channel. For patch $p$ and channel $c$, the token is
\begin{equation}
d_{p,c} =
\big(D[(p{-}1)L_p{+}1,c],\dots,D[pL_p,c]\big)
\in \mathbb{R}^{L_p},
\label{eq:token}
\end{equation}
which contains the $L_p$ consecutive deviation values in that patch. Each token is mapped to a $d$-dimensional representation using a shared linear embedding $W_e\in\mathbb{R}^{d\times L_p}$ and a learned patch-position embedding $e_p\in\mathbb{R}^{d}$,
\begin{equation}
v_{p,c}=W_e d_{p,c}+e_p .
\label{eq:embed}
\end{equation}
The gate network then scores all tokens. A transformer encoder $g_\phi$ contextualizes the token embeddings, and a linear head $W_g\in\mathbb{R}^{1\times d}$ produces a logit and opening probability for each token,
\begin{equation}
p_{p,c}=\mathrm{sig}(W_g g_\phi(v)_{p,c}),
\label{eq:logits}
\end{equation}
where $\mathrm{sig}(\cdot)$ denotes the logistic sigmoid. Since each score is computed from a contextualized representation, the gate network learns to compare tokens within the window and rank them by their forecasting relevance. This amortized scoring makes the explanation available in a single forward pass, avoiding per-instance test-time mask optimization. For datasets with many channels, we apply the encoder within each channel over its $P$ patches, rather than over all $P C$ tokens jointly. This reduces the attention cost from $O((PC)^2)$ to $O(CP^2)$ while keeping the gating procedure tractable for high-dimensional multivariate series.

\subsubsection{Stochastic hard gates.}  
To ensure that a closed token is truly dropped, we adopt binary gates. The challenge lies in optimization, since discrete gates carry no gradient. We sample each gate during training from a noisy relaxation of its logit and binarize it with a straight-through estimator \citep{l0reg2018}. Formally, we have
\begin{equation}
\begin{aligned}
m_{p,c} = \mathds{1}\big[b_{p,c} > \tfrac12\big], b_{p,c} = \mathrm{sig}\!\left(\tfrac{\ell_{p,c} + \varepsilon}{T}\right),
\end{aligned}
\label{eq:gates}
\end{equation}
where $\varepsilon = \log u - \log(1-u), u \sim \mathrm{Uniform}(0,1)$, the forward pass uses the binary $m_{p,c}$, the backward pass differentiates through the smooth $b_{p,c}$, and the temperature $T$ is annealed over training. During inference, the mask is deterministic and noise-free.

\subsubsection{Gated deviation readout.}
The readout predicts the future deviations $\hat D \in \mathbb{R}^{H\times C}$ from the masked deviations alone,
\begin{equation}
\hat D = f_{\text{base}}\big(\bar m \odot D\big),
\label{eq:readout}
\end{equation}
where $\bar m = m \otimes \mathbf{1}{L_p}$ broadcasts each gate across the $L_p$ time steps in its patch. The function $f{\text{base}}$ is a shared channel-wise feed-forward network that maps each masked deviation sequence from $L$ historical steps to $H$ future steps, jointly capturing future residual variation and local trend changes.
The final forecast re-applies the structural context, adding the future slice of the seasonal profile and inverting the frame,
\begin{equation}
\hat Y = \sigma \odot \big(\hat D + Q^{\mathrm{per}}_{[\tau+L : \tau+L+H]}\big) + \mu.
\label{eq:remap}
\end{equation}
When a token is closed, it contributes zero in deviation space. 
After the seasonal profile, scale, and level are restored, this zero deviation corresponds to the value $\mu + \sigma \odot Q$, meaning that the token follows the expected seasonal pattern at the current window level. 
The same interpretation is used in our evaluation of faithfulness. Deleting a token resets it to the seasonal profile while keeping the structural context fixed, so the architecture and evaluation metric share the same definition of an absent input. 

\begin{algorithm}[t]
\caption{\ours{} forward pass}
\label{alg:forward}
\begin{algorithmic}[1]
\REQUIRE window $X \in \mathbb{R}^{L\times C}$, phase $\tau$
\ENSURE forecast $\hat Y$, explanation $E$
\STATE $\mu, \sigma \leftarrow$ per-channel mean and std.\ of $X$
\STATE $D \leftarrow (X-\mu)/\sigma - Q^{\mathrm{per}}_{[\tau:\tau+L]}$ \hfill Eq.~\eqref{eq:decomp}
\STATE $v_{p,c} \leftarrow W_e\, d_{p,c} + e_p$ for all $(p,c)$ \hfill Eqs.~\eqref{eq:token}--\eqref{eq:embed}
\STATE $p_{p,c} \leftarrow \mathrm{sig}\big(W_g\, g_\phi(v)_{p,c}\big)$ \hfill Eq.~\eqref{eq:logits}
\IF{training}
\STATE $m \leftarrow$ hard-Concrete sample \hfill Eq.~\eqref{eq:gates}
\ELSE
\STATE $m \leftarrow \mathbb{I}\big[p > \tfrac12\big]$
\ENDIF
\STATE $\hat D \leftarrow f_{\text{base}}(\bar m \odot D)$ \hfill Eq.~\eqref{eq:readout}
\STATE $\hat Y \leftarrow \sigma \odot \big(\hat D + Q^{\mathrm{per}}_{[\tau+L:\tau+L+H]}\big) + \mu$ \hfill Eq.~\eqref{eq:remap}
\RETURN $\hat Y$ and $E = \big(m,\ \bar m \odot D\big)$, with context $C = (\tau;\,\mu,\sigma;\,Q)$
\end{algorithmic}
\end{algorithm}

\subsubsection{Training objective}
\label{sec:objective}
All parameters, comprising the profile $Q$, the gate network, and the readout, are trained end to end by minimizing the formulation:
\begin{equation}
\begin{aligned}
\mathcal{L} = \big\|\hat Y - Y\big\|_2^2
&\;+\; \beta \sum_{p,c} \mathrm{KL}\big(\mathrm{Bern}(p_{p,c})\,\|\,\mathrm{Bern}(\pi)\big)\\[2pt]
&\;+\; \lambda_b \big(\overline{p} - \rho\big)^2
\;+\; \lambda_{\text{tv}}\,\mathrm{TV}(m),
\end{aligned}
\label{eq:objective}
\end{equation}
where $\pi$ is a prior open probability, $\overline{p}$ is the mean gate probability over all tokens, $\rho \in (0,1]$ is the user-chosen budget of Problem~\ref{prob:sif}, and $\beta$, $\lambda_b$, $\lambda_{\text{tv}}$ are weighting coefficients. 
The first term is the forecasting loss. 
The second term regularizes each gate toward a sparse Bernoulli prior \citep{gsat2022}. 
The last term, $\mathrm{TV}(m) = \sum_{p,c} \lvert m_{p+1,c} - m_{p,c}\rvert$, encourages temporally contiguous selections.

\subsubsection{Explanation extraction} \label{sec:factors}

Algorithm~\ref{alg:forward} summarizes the forward pass, which emits the forecast and the explanation together. The explanation costs no additional computation compared to other methods, since every component of $E$ is the output during computing $\hat Y$, and one forward pass amounts to one gate-network encoding and one readout application. 
When a ranking over tokens is needed, the graded probabilities $p$ are emitted alongside $E$ and provide it, and point saliency for comparison with point-level explainers is $p$ broadcast over each patch's $L_p$ steps.

\begin{table*}[t]
\centering
\caption{Forecasting accuracy on average. Bold is best overall, and underline is best interpretable model. Per-horizon results in the Appendix. ProtoTS in Weather cells exceeded the compute budget (10 hours) and are omitted.}
\label{tab:accuracy}
\setlength{\tabcolsep}{2pt}
\resizebox{\textwidth}{!}{%
\begin{tabular}{l cc cc cc cc cc c cc cc cc cc}
\toprule
& \multicolumn{10}{c}{\textit{Interpretable-by-design}} & & \multicolumn{8}{c}{\textit{Black-box}} \\
\cmidrule(lr){2-11}\cmidrule(lr){13-20}
& \multicolumn{2}{c}{\textbf{\ours}} & \multicolumn{2}{c}{PatchDecomp} & \multicolumn{2}{c}{TFT} & \multicolumn{2}{c}{DLinear} & \multicolumn{2}{c}{ProtoTS$^\S$} & & \multicolumn{2}{c}{TQNet} & \multicolumn{2}{c}{CycleNet} & \multicolumn{2}{c}{iTransformer} & \multicolumn{2}{c}{PatchTST} \\
\cmidrule(lr){2-3}\cmidrule(lr){4-5}\cmidrule(lr){6-7}\cmidrule(lr){8-9}\cmidrule(lr){10-11}\cmidrule(lr){13-14}\cmidrule(lr){15-16}\cmidrule(lr){17-18}\cmidrule(lr){19-20}
Dataset & MSE & MAE & MSE & MAE & MSE & MAE & MSE & MAE & MSE & MAE & & MSE & MAE & MSE & MAE & MSE & MAE & MSE & MAE \\
\midrule
ETTh1   & \underline{\textbf{0.438}} & \underline{\textbf{0.429}} & 0.456 & 0.446 & 0.581 & 0.514 & 0.456 & 0.452 & 0.601 & 0.521 & & 0.441 & 0.434 & 0.457 & 0.441 & 0.454 & 0.448 & 0.469 & 0.455 \\
ETTh2   & \underline{0.381} & \underline{\textbf{0.402}} & 0.399 & 0.420 & 0.448 & 0.444 & 0.559 & 0.515 & 0.402 & 0.418 & & \textbf{0.378} & \textbf{0.402} & 0.388 & 0.409 & 0.383 & 0.407 & 0.387 & 0.407 \\
ETTm1   & \underline{0.383} & 0.403 & 0.384 & \underline{0.399} & 0.487 & 0.469 & 0.403 & 0.407 & 0.548 & 0.473 & & \textbf{0.377} & \textbf{0.393} & 0.379 & 0.396 & 0.407 & 0.410 & 0.387 & 0.400 \\
ETTm2   & \underline{0.272} & \underline{0.316} & 0.297 & 0.344 & 0.311 & 0.345 & 0.350 & 0.401 & 0.403 & 0.384 & & 0.277 & 0.323 & \textbf{0.266} & \textbf{0.314} & 0.288 & 0.332 & 0.281 & 0.326 \\
Weather & \underline{0.244} & \underline{0.271} & 0.253 & 0.276 & 0.275 & 0.298 & 0.265 & 0.317 & --- & --- & & \textbf{0.242} & \textbf{0.269} & 0.243 & 0.271 & 0.258 & 0.278 & 0.259 & 0.273 \\
\midrule
Overall & \underline{0.344} & \underline{\textbf{0.364}} & 0.358 & 0.377 & 0.420 & 0.414 & 0.407 & 0.418 & --- & --- & & \textbf{0.343} & \textbf{0.364} & 0.347 & 0.366 & 0.358 & 0.375 & 0.356 & 0.372 \\
\bottomrule
\end{tabular}}
\end{table*}

\section{Theoretical Analysis}
\label{sec:analysis}
Throughout the statement in previous section, the context $\mathcal{C}=(\tau;\,\mu,\sigma;\,Q)$ is computed from the window and its timestamp, Eq.~\eqref{eq:decomp} is invertible once $(\mu,\sigma,Q,\tau)$ are fixed, so $\Delta=D$, and the stochastic gates of Eq.~\eqref{eq:gates} use independent noise. The conditional information bottleneck is
\begin{equation}
\min_E \; -I(E;Y\mid\mathcal{C})+\beta I(X;E\mid\mathcal{C}),
\label{eq:cib}
\end{equation}
where $I(E;Y\mid\mathcal{C})$ is the explanation's {relevance}, and $I(X;E\mid\mathcal{C})$ is rate, the information retained from $X$.
Conditioning the bottleneck on $\mathcal{C}$ is what keeps the explanation from being calculated twice for globally predictable structure. Since $\mathcal{C}=(\tau;\mu,\sigma;Q)$ is disclosed, the saved rate is exactly $I(\mathcal{C};E)$: the explanation is not penalized for information already provided by the context. By contrast, any $\varepsilon$-sufficient unconditional explanation must spend at least $-I(\mathcal{C};Y)-\varepsilon$ bits recovering the level, scale, and cycle. 
We formalize this in Theorem~\ref{thm:rate} in the Appendix.

\begin{theorem}[Relevance controls the error floor]
\label{thm:tradeoff}
Every forecaster of the form $\hat Y=g(E;\mathcal{C})$ satisfies
$$
\mathbb{E}\|Y-\hat Y\|_2^2
\ge
\frac{n}{2\pi e}\,
\exp\!\Big(\frac{2}{n}\big(h(Y\mid\mathcal{C})-I(E;Y\mid\mathcal{C})\big)\Big),
$$
where $h(Y\mid\mathcal{C})$ is the conditional differential entropy of the future given the context. The bound decreases in the relevance $I(E;Y\mid\mathcal{C})$.
For a fixed set of opened token positions $\mathcal{M}\subseteq\{1,\dots,P\}\times\{1,\dots,C\}$, the support of the mask, write $E_{\mathcal{M}}=(\mathcal{M},\Delta_{\mathcal{M}})$ for the selective explanation that discloses the remainder values at exactly those positions. Then for nested selections $\mathcal{M}\subseteq\mathcal{M}'$,
\[
\mathbb{E}\|Y-\mathbb{E}[Y\mid E_{\mathcal{M}'},\mathcal{C}]\|_2^2
\le
\mathbb{E}\|Y-\mathbb{E}[Y\mid E_{\mathcal{M}},\mathcal{C}]\|_2^2 ,
\]
so the best achievable Bayes error is nonincreasing as the mask budget grows.
\end{theorem}
This theorem links explanation quality directly to forecasting accuracy. Once the context $\mathcal{C}$ is set, the forecast is only as good as the information contained in $E$. Higher relevance, $I(E;Y\mid\mathcal{C})$, drives down the minimum possible error. Naturally, allowing the model to see more tokens can only improve its best possible prediction, creating a clear trade-off between sparsity and accuracy controlled by $\beta$ or $\rho$.
Due to the page limitations, we provide the detailed proofs in the Appendix.

\section{Experiments}
\label{sec:experiments}

\subsubsection{Datasets.} We use standard multivariate benchmarks, \textbf{ETTh1}, \textbf{ETTh2}, \textbf{ETTm1}, \textbf{ETTm2}~\citep{informer2021}, \textbf{Weather}, 
and \textbf{Electricity}, spanning 7 to 21 channels; 
dataset statistics and splits are in Appendix~\ref{app:datasets}. We adopt the evaluation protocol of TQNet \citep{tqnet2025}, with look-back 96, horizons $\{96, 192, 336, 720\}$, and early stopping and model selection on the validation split only.

\subsubsection{Baselines.} The baselines cover both forecasting and explanation. For forecasting, we compare against two groups. (i). \emph{Self-interpretable forecasters}, including \textbf{TFT} \citep{tft2021}, \textbf{PatchDecomp} \citep{patchdecomp2026}, \textbf{DLinear} \citep{dlinear2023}, and \textbf{ProtoTS} \citep{protots2025}. (ii). \emph{Black-box state-of-the-art forecasters}, namely \textbf{TQNet} \citep{tqnet2025}, \textbf{CycleNet} \citep{cyclenet2024}, \textbf{iTransformer} \citep{itransformer2024}, \textbf{PatchTST} \citep{patchtst2023}, and \textbf{TimesNet} \citep{timesnet2023}.  For explanation, each self-interpretable forecaster is scored with its own native explanations, and we additionally compare against post-hoc explainers applied to our own trained model, including \textbf{gradient-based saliency}~\citep{saliency2014}, \textbf{Integrated Gradients (IG)}~\citealp{ig2017}, and \textbf{TIMING}~\citep{jang2025timing}. 
We further include perturbation-based patch-level occlusion, a uniform-random baseline, and two test-time mask optimization methods, \textbf{DynaMask}~\citep{dynamask2021} and \textbf{ExtremalMask}~\citep{extremalmask2023}. 
Detailed descriptions, implementation adaptations, and complete hyperparameter configurations for all baselines are provided in the Appendix.

\subsubsection{Metrics.} We follow previous paper~\citep{tqnet2025} and use \textbf{Mean Squared Error (MSE)} and \textbf{Mean Absolute Error (MAE)} to measure the forecasting accuracy. We measure faithfulness using matched-budget fidelity (ERASER; \citealp{eraser2020}). This relies on two complementary metrics based on forecast shifts: \textbf{comprehensiveness@$k$} (deleting the top-$k$ inputs, where higher is better) and \textbf{sufficiency@$k$} (keeping only the top-$k$ inputs, where lower is better). After normalizing both shifts against an all-masked baseline, the final fidelity score is simply comprehensiveness@$k$ minus sufficiency@$k$.
To ensure fair and consistent evaluation, we apply three corrections to all explainers: (i) we fix $k$ per sample to the explainer's native open count to avoid scoring unselected regions; (ii) we freeze instance-normalization stats $(\mu, \sigma)$ at their original values, treating them as disclosed structure; and (iii) we impute deleted values using the model's own reference semantics ($\mu + \sigma\, Q(t)$ or the window mean), ensuring all perturbations remain in-distribution.

\begin{table*}[t]
\centering
\caption{Matched-budget fidelity ($\rho{=}0.2$, $H{=}96$, 1024 test windows): Comp$\uparrow$, Suff$\downarrow$, Score $=$ Comp $-$ Suff, with budget $k$ fixed per sample to the native mask's open count. Bold: best per row. $^\dagger$Oracle-tuned (Appendix~\ref{app:baselines}). $^\ddagger$Explains its own model's forecast.}
\label{tab:fidelity}
\setlength{\tabcolsep}{2pt}
\resizebox{\textwidth}{!}{
\begin{tabular}{llccccccccccc}
\toprule
& & \multicolumn{7}{c}{\emph{Post-hoc}} & \multicolumn{4}{c}{\emph{Native}} \\
\cmidrule(lr){3-9}\cmidrule(lr){10-13}
Dataset & Metric & Occlusion & TIMING$^\dagger$ & IG & Saliency & Dynamask & ExtremalMask & Random & DLinear$^\ddagger$ & PatchDecomp$^\ddagger$ & TFT$^\ddagger$ & \textbf{\ours{}} \\
\midrule
\multirow{3}{*}{ETTh1} & Comp$\uparrow$  & 0.991 & 0.939 & 0.922 & 0.917 & 0.730 & 0.740 & 0.261 & 0.588 & 0.186 & 0.798 & \textbf{1.019} \\
& Suff$\downarrow$  & 0.171 & 0.279 & 0.356 & 0.335 & 0.420 & 0.449 & 0.943 & 0.201 & 0.398 & 2.276 & \textbf{0.069} \\
& Score$\uparrow$   & 0.820 & 0.660 & 0.567 & 0.581 & 0.310 & 0.292 & $-$0.682 & 0.387 & $-$0.211 & $-$1.478 & \textbf{0.950} \\
\midrule
\multirow{3}{*}{ETTm1} & Comp$\uparrow$  & 0.940 & 0.947 & 0.942 & 0.889 & 0.751 & 0.711 & 0.357 & 0.695 & 0.303 & 0.565 & \textbf{1.056} \\
& Suff$\downarrow$  & 0.303 & 0.393 & 0.396 & 0.429 & 0.535 & 0.563 & 0.904 & 0.174 & 0.468 & 0.912 & \textbf{0.133} \\
& Score$\uparrow$   & 0.637 & 0.555 & 0.547 & 0.460 & 0.216 & 0.148 & $-$0.547 & 0.521 & $-$0.165 & $-$0.347 & \textbf{0.923} \\
\midrule
\multirow{3}{*}{Weather} & Comp$\uparrow$  & 0.819 & 0.863 & 0.859 & 0.797 & 0.623 & 0.561 & 0.118 & 0.879 & 0.726 & 0.646 & \textbf{0.925} \\
& Suff$\downarrow$  & 0.116 & 0.092 & 0.091 & 0.131 & 0.235 & 0.301 & 0.709 & \textbf{0.042} & 0.318 & 0.226 & 0.077 \\
& Score$\uparrow$   & 0.703 & 0.771 & 0.768 & 0.665 & 0.389 & 0.260 & $-$0.591 & 0.838 & 0.408 & 0.420 & \textbf{0.848} \\
\bottomrule
\end{tabular}}
\end{table*}

\begin{figure*}[t]
\centering
\includegraphics[width=\textwidth]{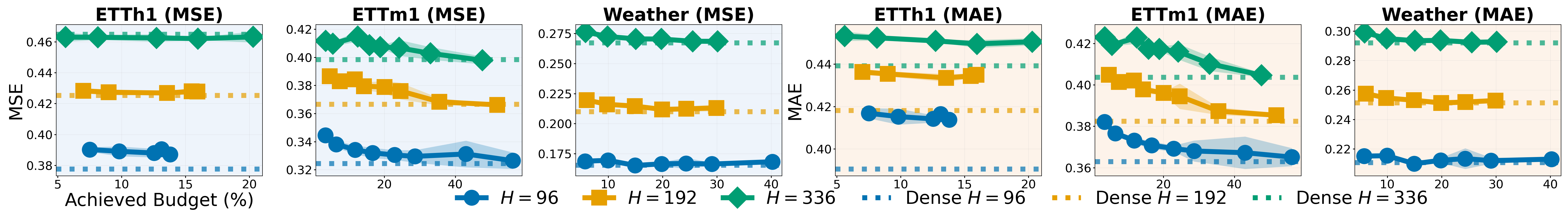}
\caption{Trade-off between explanation sparsity and forecasting error.  Test MSE (left) and MAE (right) against the measured fraction of the input read, per dataset, for horizons $H\in\{96,192,336\}$; dotted lines mark the same-horizon dense reference. Points are seed means of three seeds with a standard deviation band. }
\label{fig:pareto}
\end{figure*}

\subsubsection{Forecasting Accuracy.}
\label{sec:rq1}
As Table~\ref{tab:accuracy} lists, \ours{} is the best interpretable forecaster on every dataset in MSE and on all but ETTm1 in MAE. PatchDecomp is the closest competitor, trailing by about 4\% in overall MSE: it is essentially tied on ETTm1 but loses ground everywhere else. DLinear trails by 18\% overall; its single linear map holds up on Weather but degrades sharply on ETTh2 and ETTm2, by 47\% and 29\%, where a fixed linear read-out cannot adapt to the series. TFT trails on every dataset and by 22\% overall, so attention-based transparency buys no accuracy here. ProtoTS stays close only on ETTh2 and falls 37 to 48\% behind on the other ETT datasets. The pattern is consistent: each incumbent pays for its transparency with a capacity restriction that at least one dataset exposes, whereas \ours{} is the only model in the group whose accuracy stays on the state-of-the-art frontier across all five. Detailed per-horizon results and analysis, including PatchDecomp's short- versus long-horizon split and ProtoTS's long-horizon collapse, are given in Appendix.

Comparing against black-box state-of-the-art forecasters, in overall MSE, \ours{} is statistically tied with TQNet ($\Delta{=}0.2\%$, within seed noise) and ahead of CycleNet, PatchTST, and iTransformer. In overall MAE, it shares the best score with TQNet exactly. At the dataset level, it is the best model outright on ETTh1 and stays within 2.5\% of the best black-box model everywhere else, so no dataset exposes an interpretability tax. In the Appendix, we provide per-horizon numbers, seed variability, configuration details.

\subsubsection{Faithfulness under matched-budget fidelity.}
\label{sec:rq3}
In Table~\ref {tab:fidelity}, we report matched-budget fidelity on ETTh1, ETTm1, and Weather, at open rates of 13.9\%, 16.8\%, and 20.4\%, respectively. The proposed native mask achieves the highest Score on all three datasets. It combines high comprehensiveness with low sufficiency, showing that the selected inputs both capture the forecast shift and preserve the prediction.

Among post-hoc explainers, Occlusion is the strongest baseline, but it still trails the native mask by 0.13 to 0.29 in Score. Compared with IG and Saliency, it suggests that its estimator helps within the gradient-based family. 
Dynamask and ExtremalMask deserve a closer look. Despite 500 optimization steps per test batch and the same or an oracle budget,  their sufficiency ranges from 0.24 to 0.56 across the three datasets against 0.07 to 0.13 for the native mask, and their best Score is 0.39 against the native 0.85 to 0.95. 
This indicates that per-instance search cannot recover the sparse computation installed by end-to-end training.

The native baselines show that architectural transparency alone is not enough.
TFT and PatchDecomp obtain negative Scores on two of the three datasets, meaning that keeping their top credited inputs often fails to reproduce their own forecasts. DLinear provides a complementary case. Its attribution is mathematically exact, since it is computed from the linear weights in closed form, and its accuracy matches the published DLinear results within 2\%. But at the same budget, scores vary from 0.39 to 0.84, indicating influence is spread across the full window rather than concentrated in a sparse subset.

\subsubsection{Accuracy and explanation trade-off.}
\label{sec:rq4}
To measure this tradeoff, we sweep the budget $\rho$ from 0.05 to 0.5, retrain one model at each setting, and plot test MSE and MAE against the achieved budget in Figure~\ref{fig:pareto}.  At the budgets used in our main experiments, the mask reads only 13 to 20\% of the input, yet the accuracy loss remains small. The MSE gap from the full input baseline is at most 3.4\% on ETTh1, 3.5\% on ETTm1, and 2\% on Weather, while MAE stays within about 6\%. The gap shrinks at longer horizons and with looser budgets.

The frontier is different across datasets. On ETTm1, as the allowed read rate increases, both MSE and MAE improve smoothly, from about $+6\%$ error at a 3\% read to roughly the full input baseline near a 50\% read.  ETTh1 and Weather behave differently. Their unconstrained models already ignore much of the window, using about 20\% of the input on ETTh1 and 40 to 50\% on Weather. MSE and MAE also reveal slightly different behavior. For instance, on ETTh1, MSE suggests that sparsity is nearly free, including a 0.4 to 0.6\% gain at $H{=}336$,  while MAE still shows a 3 to 6\% penalty. 
Overall, sparse forecasting preserves accuracy with little differences while making the model's input dependence explicit.

\begin{table}[t]
\centering
\caption{Ablation study on ETTh1. \emph{Open} is the realized base-mask open rate. $^\ddagger$Fidelity is vacuous when the mask saturates fully open ($100\%$), as it then excludes nothing.}
\label{tab:ablations}
\setlength{\tabcolsep}{1pt}
\begin{tabular}{lcccccc}
\toprule
& MSE$\downarrow$ & MAE$\downarrow$ & Open & Comp$\uparrow$ & Suff$\downarrow$ & Score$\uparrow$ \\
\midrule
\textbf{\ours{}} & 0.388 & 0.411 & 0.136 & 1.030 & 0.071 & {0.959} \\
\midrule
\multicolumn{7}{@{}l}{\emph{Modules}}\\
$-$ gated base & 0.382 & 0.394 & 0.000 & 0.000 & 1.000 & $-$1.000 \\
$-$ hard gates & 0.378 & 0.407 & 0.099 & 0.870 & 0.249 & 0.621 \\
$-$ cycle $Q$ & 0.428 & 0.430 & 0.242 & 1.009 & 0.011 & 0.999 \\
\midrule
\multicolumn{7}{@{}l}{\emph{Loss terms}}\\
$-$ budget & 0.380 & 0.392 & 1.000 & 1.000 & 0.000 & 1.000$^\ddagger$ \\
$-$ gate TV & 0.387 & 0.411 & 0.135 & 1.028 & 0.071 & 0.957 \\
$-$ diversity & 0.391 & 0.412 & 0.145 & 1.058 & 0.096 & 0.961 \\
\bottomrule
\end{tabular}
\end{table}

\subsubsection{Case study.}
\label{sec:casestudy}

Figure~\ref{fig:gatemap} shows a complete explanation for one ETTh1 test window from $\rho{=}0.2$. Each row represents one channel, so all model quantities are aligned on a shared time axis.
For this example, the model opens only $16\%$ of the (patch $\times$ channel) tokens. 
These open gates are the only window-specific values the deviation readout is allowed to use, so the highlighted regions are not a post-hoc saliency estimate but the actual evidence used by the forecast (Problem~\ref{prob:sif}). 
Channels whose behavior is already explained by the periodic profile, the current level, and the current scale can remain fully closed. 
The open gates, in turn, mark where the window deviates from this disclosed structure. These masked deviations $\bar m \odot D$ are all that the readout of Eq.~\eqref{eq:readout} consumes.

\subsubsection{Ablation study}
\label{sec:rq5}
In Table~\ref{tab:ablations}, we conduct ablation studies on the architectural modules and the loss terms on ETTh1 at $H{=}96$. 
As the results show, removing the gated base or the hard gates barely moves MSE and MAE, which both stay within $2.5\%$ of the full model, yet either one destroys the explanation. 
Without gating, the mask receives no gradient from the forecast and collapses fully closed ($0\%$ open, score $-1.0$), so the reported saliency is empty. With soft gates the mask leaks a fraction of every token, so keeping its top-$k$ no longer reproduces the forecast and fidelity falls from $0.96$ to $0.62$. 
Dropping the budget term leaves accuracy untouched but saturates the mask fully open ($100\%$); its fidelity score of $1.0$ is vacuous, since a mask over the entire input excludes nothing and matched-budget fidelity is then trivially maximal. 
The gated base and the budget are complementary and together they force the mask to be a {selective} subset.  Removing the gate-TV or the factor-diversity penalty has little influence on fidelity.
The recurrent cycle $Q$ is the most important part for forecasting. Removing it will cause 10\% MSE performance drop. Beyond these on/off ablations, we also provide detailed hyperparameter analysis in the Appendix.

\begin{figure}[t]
\centering
\includegraphics[width=\columnwidth]{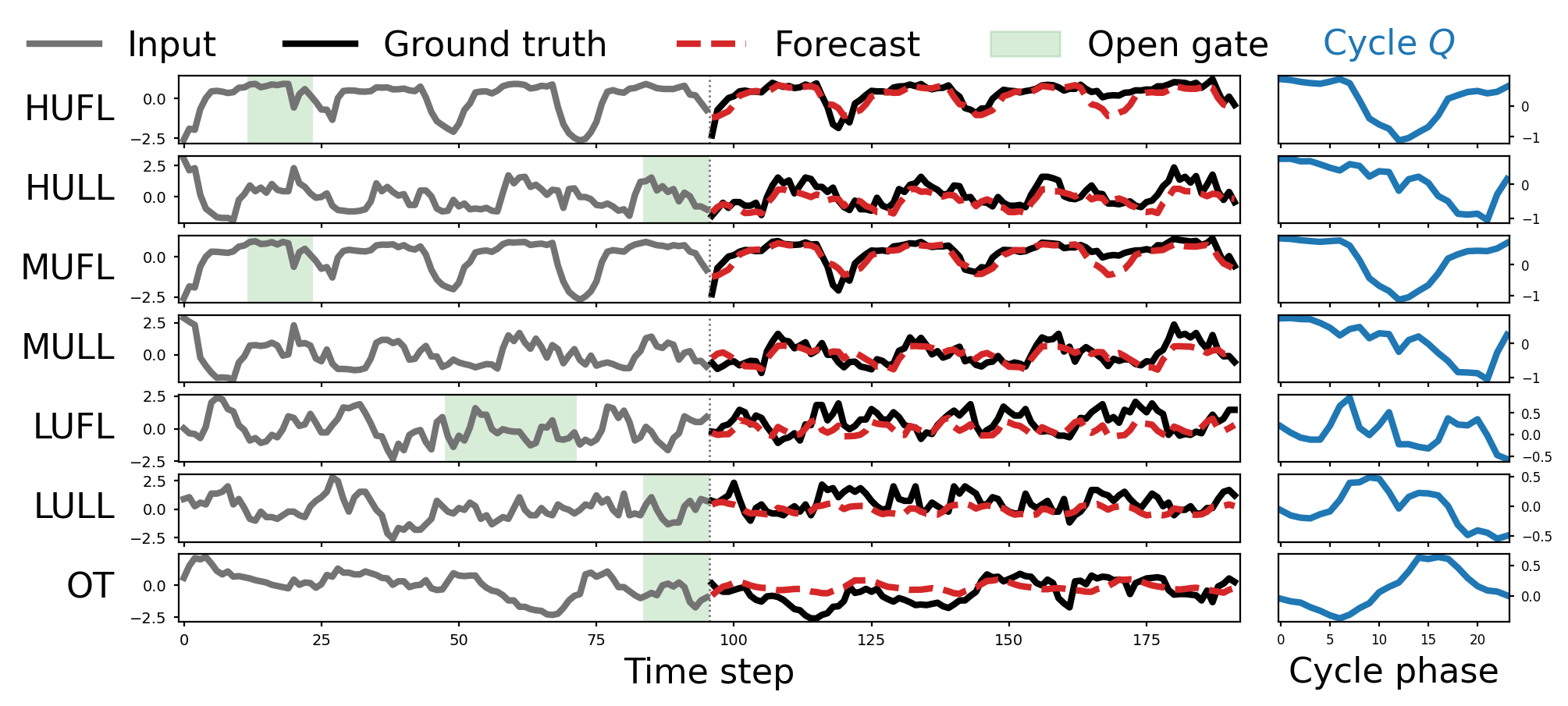}
\caption{An example of \ours{} explanation for one ETTh1 test window ($\rho{=}0.2$), one row per channel in RevIN-normalized space. Each row aligns, on a shared time axis, the lookback input (gray), the ground truth (black) and forecast (red dashed), with the mask's open gates (shaded green). The narrow right panel is the learned periodic profile $Q$.}
\label{fig:gatemap}
\end{figure}

\section{Conclusion}
\label{sec:conclusion}

We presented {\ours}, a multivariate forecaster with faithful explanations at no accuracy cost. Its budgeted stochastic mask selects deviation tokens, while disclosed level, scale, and seasonal context reconstruct the forecast; thus, the selected tokens are the only window-specific input pathway.
{\ours} matches state-of-the-art black-box accuracy, achieves stronger matched-budget fidelity than existing interpretable forecasters, and outperforms test-time optimizers on their own objective at zero inference cost. Together, certified explanations and matched-budget evaluation provide a rigorous foundation for interpretable forecasting.

\bibliography{aaai2027}
\clearpage
\newpage

\newpage
\appendix

\section{Theoretical Analysis: Full Proofs}
\subsection{Assumptions}
This appendix proves the results of Section Theoretical Analysis. We first record, as (A1)--(A3), the standing assumptions stated in prose in the Section Preliminaries and at the Section Theoretical Analysis. All results below are understood under these assumptions. Throughout, $n=HC$ denotes the dimension of the future $Y$, and all differential entropies are assumed finite.
\begin{itemize}
\item \textbf{(A1)} \emph{Determined context.} $\mathcal{C}=c(X)=(\tau;\,\mu,\sigma;\,Q)$ is a measurable function of the input window and its timestamp.
\item \textbf{(A2)} \emph{Invertible decomposition.} Once $(\tau, \mu,\sigma,Q)$ are fixed, the decomposition of Eq.~\eqref{eq:decomp} is invertible, so given $\mathcal{C}$ the window $X$ and its remainder $\Delta=D$ determine one another. The structural frame extends periodically over the forecast horizon (Eq.~\eqref{eq:remap}), so given $\mathcal{C}$ the future $Y$ and its remainder $\Delta_Y$ likewise determine one another.
\item \textbf{(A3)} \emph{Exogenous gate noise.} The explanation is generated from the remainder and the context together with independent noise, $E=e(\Delta,\mathcal{C},U)$, where $U$ collects the logistic noises of the sampler in Eq.~\eqref{eq:gates} and is independent of $(X,Y)$.
\end{itemize}

\subsection{Lemma}
\noindent To provide the analysis, we first introduce an auxiliary lemma for the proofs.
\begin{lemma}[Conditional Markov structure]
\label{lem:markov}
Suppose that assumptions \textnormal{(A1)}--\textnormal{(A3)} hold, we have the following three  properties:
\begin{itemize}
\item The chain $Y \leftrightarrow \Delta \leftrightarrow E$ holds given $\mathcal{C}$, hence $I(E;Y\mid\mathcal{C})\le I(\Delta;Y\mid\mathcal{C})$. 
\item  For fixed selective explanations $E_{\mathcal{M}}=(\mathcal{M},\Delta_{\mathcal{M}})$ with $\mathcal{M}\subseteq\mathcal{M}'$, $I(E_{\mathcal{M}};Y\mid\mathcal{C})\le I(E_{\mathcal{M}'};Y\mid\mathcal{C})$.
\item Relevance to the future equals relevance to its deseasonalized remainder, $I(E;Y\mid\mathcal{C})=I(E;\Delta_Y\mid\mathcal{C})$.
\end{itemize}
\end{lemma}

\begin{proof}
By (A3), $E$ is generated from $(\Delta,\mathcal{C})$ and exogenous noise independent of $Y$. Hence $Y\perp E\mid(\Delta,\mathcal{C})$, which is the conditional Markov chain. Applying the chain rule to $I((E,\Delta);Y\mid\mathcal{C})$ in both orders yields
\[
I(E;Y\mid\mathcal{C})+I(\Delta;Y\mid E,\mathcal{C})
=I(\Delta;Y\mid\mathcal{C})+I(E;Y\mid\Delta,\mathcal{C}).
\]
The last term is zero by conditional independence and the second term is nonnegative, so $I(E;Y\mid\mathcal{C})\le I(\Delta;Y\mid\mathcal{C})$.
If $\mathcal{M}\subseteq\mathcal{M}'$, then $E_{\mathcal{M}}$ is a deterministic projection of $E_{\mathcal{M}'}$. The same data-processing argument gives
$I(E_{\mathcal{M}};Y\mid\mathcal{C})\le I(E_{\mathcal{M}'};Y\mid\mathcal{C})$.

For the last identity, (A2) extends the determined structure over the horizon, so for each fixed context value $c$ the subtraction $Y\mapsto\Delta_Y$ is a bijection with measurable inverse. Data processing applied in both directions shows that an invertible transformation of one argument leaves mutual information unchanged, so $I(E;Y\mid\mathcal{C}=c)=I(E;\Delta_Y\mid\mathcal{C}=c)$ for every $c$; averaging over $\mathcal{C}$ gives $I(E;Y\mid\mathcal{C})=I(E;\Delta_Y\mid\mathcal{C})$.
\end{proof}

\noindent We first restate the exact rate-saving result, which motivates conditioning the information bottleneck on the disclosed context.
\begin{theorem}[Exact rate saving from disclosing the context]
\label{thm:rate}
For any explanation $E$,
$ I(X;E)=I(\mathcal{C};E)+I(\Delta;E\mid\mathcal{C}), $
so the rate saved by disclosing the context is exact:
$$ \Delta R(E)\;:=\;I(X;E)-I(\Delta;E\mid\mathcal{C})\;=\;I(\mathcal{C};E)\;\ge\;0,$$
with equality if and only if $E$ is independent of $\mathcal{C}$. Moreover, the saving is at least the predictive information routed through the context,
$ I(\mathcal{C};E)\;\ge\;I(E;Y)-I(E;Y\mid\mathcal{C})$.
\end{theorem}
\noindent Conditioning on $\mathcal{C}$ prevents double-counting: subtracting $I(\mathcal{C};E)$ from the rate means the explanation is not penalized for information already present in the context, while the encoder remains free to adapt to $\mathcal{C}$. The final inequality exposes the flaw of unconditional explanations, which to obtain predictive accuracy ($I(E;Y)\ge I(X;Y)-\varepsilon$) must waste a portion of the rate ($I(\mathcal{C};Y)-\varepsilon$) reproducing the shared structure.

\subsection{Proof of Theorem ~\ref{thm:rate}}
\begin{proof} 
By (A1), $\mathcal{C}=c(X)$ is a measurable function of $X$, so the context is degenerate conditionally on $X$ and mutual information with it vanishes:
\[
I((X,\mathcal{C});E)=I(X;E)+I(\mathcal{C};E\mid X)=I(X;E).
\]
Expanding the same quantity with $\mathcal{C}$ first,
\[
I((X,\mathcal{C});E)=I(\mathcal{C};E)+I(X;E\mid\mathcal{C}).
\]
By (A2), for each fixed context value $c$ the determined structure is a fixed quantity, so the map $\varphi_c:X\mapsto\Delta$ that subtracts it is a measurable bijection with measurable inverse $\varphi_c^{-1}:\Delta\mapsto X$. 
Deterministic transformations of one argument cannot increase mutual information, applied in both directions: $\Delta=\varphi_c(X)$ gives $I(\Delta;E\mid\mathcal{C}=c)\le I(X;E\mid\mathcal{C}=c)$, and $X=\varphi_c^{-1}(\Delta)$ gives the reverse inequality, hence equality for every $c$. Averaging over $\mathcal{C}$,
\[
I(X;E\mid\mathcal{C})=I(\Delta;E\mid\mathcal{C}).
\]
Combining the previous steps proves the decomposition, and therefore
\[
\Delta R(E)=I(X;E)-I(\Delta;E\mid\mathcal{C})=I(\mathcal{C};E)\ge0,
\]
with equality if and only if $E\perp\mathcal{C}$, since mutual information vanishes exactly under independence.

Furthermore, expand $I(E;(\mathcal{C},Y))$ by the chain rule in both orders:
\begin{align*}
I(E;\mathcal{C})+I(E;Y\mid\mathcal{C})
& =I(E;(\mathcal{C},Y)) \\
& =I(E;Y)+I(E;\mathcal{C}\mid Y).
\end{align*}
Rearranging and using the nonnegativity of conditional mutual information,
\begin{align*}
I(E;\mathcal{C})
& =I(E;Y)-I(E;Y\mid\mathcal{C})+I(E;\mathcal{C}\mid Y) \\
& \ge I(E;Y)-I(E;Y\mid\mathcal{C}),
\end{align*}
with equality if and only if $E\perp\mathcal{C}\mid Y$.

Repeating Steps 1--3 with $Y$ in place of $E$ gives $I(X;Y)=I(\mathcal{C};Y)+I(\Delta;Y\mid\mathcal{C})$. If $E$ is $\varepsilon$-sufficient, $I(E;Y)\ge I(X;Y)-\varepsilon$, then combining Step 4 with the ceiling $I(E;Y\mid\mathcal{C})\le I(\Delta;Y\mid\mathcal{C})$ of Lemma~\ref{lem:markov},
\begin{align*}
I(\mathcal{C};E)
& \;\ge\; I(E;Y)-I(E;Y\mid\mathcal{C}) \\
& \;\ge\; \big(I(X;Y)-\varepsilon\big)-I(\Delta;Y\mid\mathcal{C}) \\
& \;=\; I(\mathcal{C};Y)-\varepsilon.    
\end{align*}
This is the claim in the discussion following Theorem~\ref{thm:rate}: to be $\varepsilon$-sufficient, an unconditional explanation must spend at least $I(\mathcal{C};Y)-\varepsilon$ of its rate reproducing structure already disclosed by the context.
\end{proof}

\subsection{Proof of Theorem ~\ref{thm:tradeoff}}
\begin{proof} 
\emph{Part 1 (the error floor).} For fixed $(e,c)$, the conditional mean $\bar y_{e,c}=\mathbb{E}[Y\mid E=e,\mathcal{C}=c]$ minimizes squared error, so every predictor $g(e;c)$ satisfies
\[
\mathbb{E}[\|Y-g(e;c)\|_2^2\mid e,c]\ge \operatorname{tr}\Sigma_{e,c},
\]
where $\Sigma_{e,c}$ is the conditional covariance of $Y$. The Gaussian maximizes entropy among distributions with covariance $\Sigma_{e,c}$, and the arithmetic--geometric mean inequality gives
\[
h(Y\mid e,c)
\le \frac{n}{2}\log\!\left(2\pi e\,\frac{\operatorname{tr}\Sigma_{e,c}}{n}\right).
\]
Equivalently,
\[
\operatorname{tr}\Sigma_{e,c}
\ge
\frac{n}{2\pi e}\exp\!\left(\frac{2}{n}h(Y\mid e,c)\right).
\]
Taking expectations over $(E,\mathcal{C})$ and applying Jensen's inequality to the convex exponential,
\[
\mathbb{E}\|Y-\hat Y\|_2^2
\ge
\frac{n}{2\pi e}\exp\!\left(\frac{2}{n}h(Y\mid E,\mathcal{C})\right).
\]
Finally,
$h(Y\mid E,\mathcal{C})=h(Y\mid\mathcal{C})-I(E;Y\mid\mathcal{C})$, which gives the stated bound.

\emph{Part 2 (monotonicity in the budget).} Because $E_{\mathcal{M}}$ is measurable with respect to $E_{\mathcal{M}'}$, every predictor using $(E_{\mathcal{M}},\mathcal{C})$ is also a valid predictor using $(E_{\mathcal{M}'},\mathcal{C})$. The conditional expectation
$\mathbb{E}[Y\mid E_{\mathcal{M}'},\mathcal{C}]$ is the $L^2$ projection of $Y$ onto the larger predictor space, and projection error cannot increase when the space is enlarged. This gives the displayed inequality; since any larger budget admits every selection feasible under a smaller one, the minimal Bayes risk is nonincreasing in the budget.
\end{proof}

\section{Dataset statistics}
\label{app:datasets}

Table~\ref{tab:datasets} summarizes the used benchmarks, all standard public multivariate long-term-forecasting datasets:
\begin{itemize}
\item \emph{ETTh1, ETTh2, ETTm1, ETTm2} \citep{informer2021} record two years (July 2016 to July 2018) of Electricity Transformer Temperature data from two transformers in two Chinese counties. Each has 7 channels, the target oil temperature plus six external power-load features; the \texttt{h} variants are sampled hourly and the \texttt{m} variants every 15 minutes.
\item \emph{Weather} contains 21 meteorological indicators, such as air temperature, humidity, and wind speed, recorded every 10 minutes throughout 2020 at the Max Planck Institute for Biogeochemistry weather station in Jena, Germany.
\end{itemize}

All splits and preprocessing follow the TQNet protocol verbatim: the ETT family is split 6:2:2 and the remaining datasets 7:1:2 chronologically; every channel is z-score standardized with statistics computed on the training split only, and all metrics are computed on the standardized data. Windows are formed by sliding over each split with look-back $L{=}96$ and horizons $H\in\{96,192,336,720\}$, and each window's timestamp provides the phase index $\tau$ of Eq.~\eqref{eq:decomp}. The cycle length $W$ of the seasonal profile $Q$ is fixed from the sampling rate, following the CycleNet convention \citep{cyclenet2024}: one day for the ETT family,
and Weather
($W{=}24$, $96$, $144$, and $144$ respectively).

\section{Baseline details}
\label{app:baselines}
This section describes each baseline. The information contains what the method is, what it emits as an explanation where one exists, and how it is configured or adapted for our protocol.
\paragraph{Black-box forecasters.}
The black-box group establishes the accuracy state of the art; these models offer no native explanation, so they appear in the accuracy comparison only, with their numbers taken from the TQNet consolidated table under the identical fixed-lookback protocol rather than rerun.
\begin{itemize}
\item \emph{TQNet} \citep{tqnet2025} captures multivariate correlations with periodically shifted learnable temporal queries in a single-layer attention over the raw input, followed by a lightweight MLP; it is the protocol anchor whose training pipeline our accuracy evaluation reuses.
\item \emph{CycleNet} \citep{cyclenet2024} models the recurring periodic pattern of each channel with learnable recurrent cycles and forecasts the residual with a linear or MLP backbone; our learned seasonal profile $Q$ adopts this mechanism, as credited in \S\ref{sec:method}.
\item \emph{iTransformer} \citep{itransformer2024} inverts the Transformer: each channel's whole look-back is embedded as one variate token, attention runs across variate tokens to model channel dependence, and the feed-forward network models the temporal dimension.
\item \emph{PatchTST} \citep{patchtst2023} is a channel-independent Transformer over patch tokens, with each channel processed as a separate sequence of sub-series patches under shared weights.
\item \emph{TimesNet} \citep{timesnet2023} folds the 1D series into a set of 2D tensors along its dominant periods and models intra- and inter-period variation with inception-style 2D convolutions.
\end{itemize}

\begin{table}[t]
\centering
\caption{Statistics of the multivariate benchmarks.}
\label{tab:datasets}
\setlength{\tabcolsep}{2pt}
\begin{tabular}{lcccc}
\toprule
Dataset & Channels & Time steps & Frequency & Split \\
\midrule
ETTh1, ETTh2 & 7   & 17{,}420 & 1 hour  & 6:2:2 \\
ETTm1, ETTm2 & 7   & 69{,}680 & 15 min  & 6:2:2 \\
Weather      & 21  & 52{,}696 & 10 min  & 7:1:2 \\
Electricity  & 321 & 26{,}304 & 1 hour  & 7:1:2 \\
Solar        & 137 & 52{,}560 & 10 min  & 7:1:2 \\
Traffic      & 862 & 17{,}544 & 1 hour  & 7:1:2 \\
\bottomrule
\end{tabular}
\end{table}

\paragraph{Self-interpretable forecasters.}
These models are trained on the same data as \ours{} and are the direct competitors. Each is scored on forecasting accuracy, and its native explanation, where one exists, is scored on the fidelity metrics against its own forecasts.
\begin{itemize}
\item \emph{TFT} \citep{tft2021} forecasts with variable-selection networks and interpretable attention. We implement its interpretable core faithfully: a per-timestep variable-selection network of gated residual networks that emits softmax weights over channels, an LSTM encoder, and an interpretable multi-head attention in which all heads share a single value projection so that the head-averaged attention matrix is directly readable. Components absent from our benchmarks, namely static covariates, known-future inputs, and the quantile head, are omitted; the model is trained with MSE like every other model in the harness, and instance normalization is added, without which TFT-era models underfit the distribution shift of the long-horizon benchmarks. Its native importance map is the product of the temporal attention, averaged over heads and horizon steps, and the per-timestep variable-selection weight, exactly the two quantities its paper reads off as interpretability outputs.
\item \emph{PatchDecomp} \citep{patchdecomp2026} writes the forecast as an exact sum of per-patch contributions. Our reimplementation matches the official architecture: a per-channel convolutional patch embedding with learned positional context, residual MLP encoders, a per-patch scalar bias, a single bias-free multi-head attention with future-patch queries over input-patch keys and values, and a bias-free linear decoder. Because every map after the attention is linear and bias-free, the forecast decomposes exactly into additive per-input-patch contributions, which are its native explanation; the importance of a (patch, channel) token is its mean absolute contribution over the horizon, broadcast to point resolution. Two differences from the official code are disclosed: our benchmarks carry no static or known-future covariates, so those encoders are dropped, and the real-data mode is channel-independent with shared weights, which is how its own LTSF experiments treat multivariate data.
\item \emph{DLinear} \citep{dlinear2023} decomposes the input into a moving-average trend and a remainder and forecasts each with a per-channel linear map. The model is fully transparent, and its native attribution is exact in closed form, since each future value is a weighted sum of past values with known weights.
\item \emph{ProtoTS} \citep{protots2025} explains by similarity to hierarchical prototypes, which has no input-attribution analogue, so it joins the accuracy comparison only. It is run from its official release, channel-independently under the anchor protocol, which is the one adaptation its multivariate application requires, keeping its two-stage training and prototype-splitting schedule with release hyperparameters and best-validation selection throughout. Its per-cell training cost grows quickly with channels and horizon: the Weather cells exceeded a 10-hour single-GPU budget mid-training and are excluded from Table~\ref{tab:accuracy} rather than reported from partially trained models, and the ETTm2 $H{=}720$ cell hit the same cap at epoch 56 of 60 with its learning rate at the floor and validation flat for several epochs, so it is reported from its best-validation checkpoint, the same selection the completed run would have used.
\end{itemize}
TFT and PatchDecomp receive per-dataset validation tuning identical to our own model's, and each native explanation is scored on its own model where fidelity is concerned.

\paragraph{Gradient-based post-hoc explainers.}
All three explain our trained model. Because the deployed model thresholds its gates at inference and gradients do not flow through hard gates, the gradient explainers are computed on the model's soft relaxation, with gates equal to their sigmoid probabilities, which is the standard smooth-surrogate assumption these methods make. The attribution target is the scalar forecast energy $\tfrac12\|\hat Y\|_2^2$.
\begin{itemize}
\item \emph{Saliency} \citep{saliency2014} is the absolute input gradient of the target.
\item \emph{Integrated Gradients} (IG; \citealp{ig2017}) integrates gradients along the straight path from a baseline to the input over 32 steps; the baseline is the per-channel window mean, an in-distribution reference, and the attribution is the path-averaged gradient times the input-minus-baseline displacement.
\item \emph{TIMING} \citep{jang2025timing} is temporality-aware IG: at each interpolation step, random contiguous time segments, one random channel each, are held at the true input and excluded from both the interpolation and the gradient, so every gradient is taken at a point that is on-manifold in the fixed regions, and the path sum is normalized per cell by how often that cell was free. Our port follows the reference implementation and uses the same target and baseline as the IG row, so the difference between the two rows isolates the estimator. Its release defaults, 50 segments of minimum length 10 tuned for MIMIC-III's 48-step windows, saturate a 96-step window; the TIMING row of Table~\ref{tab:fidelity} therefore reports the best of the defaults and two segment-count/length variants per dataset, which is favorable treatment, and the defaults alone score 0.404/0.450/0.771 on ETTh1/ETTm1/Weather.
\end{itemize}

\begin{table*}[t]
\centering
\caption{Adopted \ours{} configuration on all datasets. All cells share a fixed backbone: hard straight-through gates (temperature $0.5$, annealed), gate width $d{=}64$, patch length $12$, RevIN on, per-gate prior $\pi{=}0.5$, and $30$ Adam epochs.
} 
\label{tab:config}
\small
\begin{tabular}{l ccccc c}
\toprule
Dataset & Readout & Gate layers & Learning rate & $\beta$ & Cycle $W$ & Cycle used \\
\midrule
ETTh1   & gated-linear            & 1 & $1{\times}10^{-2}$          & $0.02$ & 24  & yes \\
ETTh2   & gated-linear            & 1 & $1{\times}10^{-2}$          & $0.02$ & 24  & yes\\
ETTm1   & gated-linear            & 2 & $1{\times}10^{-2}$          & $0.02$ & 96  & yes \\
ETTm2   & gated-linear            & 2 & $1{\times}10^{-2}$          & $0.02$ & 96  & yes \\
Weather & MLP                     & 1 & $3{\times}10^{-3}$          & $0.02$ & 144 & yes \\
\midrule
Electricity & MLP & 1 & $3{\times}10^{-3}$ & $0.02$ & 168 & yes \\
Solar       & MLP & 1 & $3{\times}10^{-3}$ & $0.1$  & 144 & yes \\
Traffic     & MLP & 1 & $3{\times}10^{-3}$ & $0.02$ & 168 & yes \\
\bottomrule
\end{tabular}
\end{table*}

\paragraph{Perturbation-based post-hoc explainers.}
\emph{Occlusion} is gradient-free and runs on the real hard-gated model: each (patch, channel) block is replaced by the per-channel reference, the squared forecast shift is recorded as that block's importance, and the score is broadcast back to point resolution; the block length equals the model's own patch length. \emph{Random} assigns uniform-random importance and serves as the floor of the fidelity metrics.

\paragraph{Test-time mask optimizers.}
Dynamask \citep{dynamask2021} and ExtremalMask \citep{extremalmask2023} learn a per-instance mask by optimizing directly against the model's own output, so they are near-oracles on fidelity metrics by construction. Both are published for classification; we re-target them to forecasting by replacing the class-probability-preservation term of their objectives with forecast preservation, the MSE over the full $H\times C$ forecast, with gradients through the same soft relaxation as above.
\begin{itemize}
\item \emph{Dynamask} optimizes per-instance mask logits with Adam under its published extremal keep-ratio constraint, the sorted-mask regularizer, which lets us pin the mask to the same budget as the native mask, together with its temporal total-variation term; the budget pressure is annealed from soft to binding over the run. Its published perturbation operator fades masked values to a per-channel moving average; when scoring, we instead hand it the evaluation protocol's own replacement reference, so it optimizes against exactly the operator the metric perturbs with, the most favorable setting.
\item \emph{ExtremalMask} keeps the same budget machinery but learns the replacement values jointly with the mask: a small convolutional network proposes replacements around the neutral anchor and is $\ell_2$-regularized toward it, so the mask cannot hide information in the perturbation.
\end{itemize}
Both run 500 optimization steps per test batch with the same or an oracle sparsity budget, and \S\ref{sec:rq3} verifies their scores are converged rather than under-optimized.

\section{Implementation}
The gate network is 1 to 2 Transformer layers with $d{=}64$ and patch length 12 on real data, using channel-independent gate attention for large $C$. The base readout is a gated linear map or MLP chosen on validation. Budget sparsity is $\rho{=}0.2$ for explanation studies, and $\beta{=}0.02$ KL elsewhere. Optimization uses Adam, with the gate temperature $T$ of Eq.~\eqref{eq:gates} annealed over training. Headline results use the protocol seed 2024, and for robustness every headline cell is additionally run with 3 seeds, with a seed standard deviation of 0.005 MSE or less on 17 of 20 cells and a maximum of 0.010. All real-data training fits on a single RTX 4090 GPU. We further provide all the hyperparameter settings in Table~\ref{tab:config}.

\section{Factor decomposition}
\label{app:factors}

For controlled studies, the gate head emits $K{+}1$ maps. Map $0$ is the base mask of the main text, and maps $1$ through $K$ feed a masked-replacement pooling into factor scores $s_k$ and per-channel horizon loadings, $\hat D^{\text{struct}}_{h,c} = \sum_k A_{h,k,c}\, s_k$, giving a horizon-resolved account of which support drives which part of the future. On the synthetic benchmarks with planted multi-driver structure, this decomposition is recoverable. On the real benchmarks, the factor path collapses and contributes exactly zero, since zeroing $A$ changes neither accuracy nor any faithfulness score (Table~\ref{tab:ablations}), so the real-data model descopes it and the deployed explanation.

\section{Global gate-usage statistics}
\label{app:gatefreq}

Figure~\ref{fig:gatefreq} aggregates the mask over 512 ETTh1 test windows, giving the gate-open frequency per (lag-patch, channel). The daily structure, at lag $-96$ to $-84$ and the most recent patches, and the exclusion of near-constant channels emerge without supervision, which is the population-level counterpart of the single-window case study in Figure~\ref{fig:gatemap}.

\begin{figure}[t]
\centering
\includegraphics[width=\columnwidth]{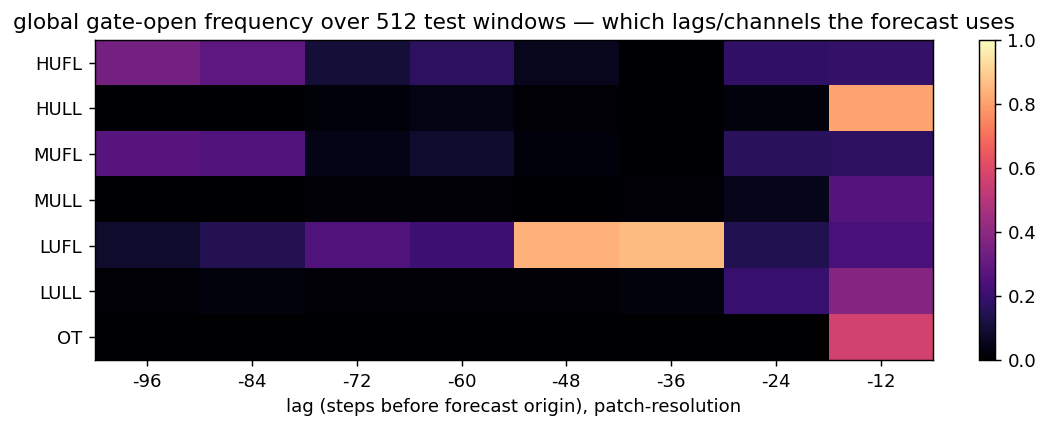}
\caption{Gate-open frequency per (lag-patch, channel) over 512 ETTh1 test windows ($\rho{=}0.2$, $H{=}96$).}
\label{fig:gatefreq}
\end{figure}

\begin{table*}[t]
\centering
\caption{Per-horizon forecasting results under the TQNet protocol (look-back 96, standardized metrics), extending Table~\ref{tab:accuracy}: each Avg row is the mean of the four horizons above it and reproduces the corresponding row of Table~\ref{tab:accuracy}. Bold marks the best overall per column pair, and underline marks the best interpretable model. $^\S$ProtoTS Weather cells exceeded the 10-hour-per-cell compute budget and are omitted, so its Average row is also omitted.}
\label{tab:perhorizon}
\setlength{\tabcolsep}{2.6pt}
\resizebox{\textwidth}{!}{
\begin{tabular}{ll cc cc cc cc cc cc cc cc cc cc}
\toprule
& & \multicolumn{10}{c}{\textit{Interpretable-by-design}} & \multicolumn{10}{c}{\textit{Black-box}} \\
\cmidrule(lr){3-12}\cmidrule(lr){13-22}
& & \multicolumn{2}{c}{\textbf{\ours}} & \multicolumn{2}{c}{PatchDecomp} & \multicolumn{2}{c}{TFT} & \multicolumn{2}{c}{DLinear} & \multicolumn{2}{c}{ProtoTS$^\S$} & \multicolumn{2}{c}{TQNet} & \multicolumn{2}{c}{CycleNet} & \multicolumn{2}{c}{iTransformer} & \multicolumn{2}{c}{PatchTST} & \multicolumn{2}{c}{TimesNet} \\
\cmidrule(lr){3-4}\cmidrule(lr){5-6}\cmidrule(lr){7-8}\cmidrule(lr){9-10}\cmidrule(lr){11-12}\cmidrule(lr){13-14}\cmidrule(lr){15-16}\cmidrule(lr){17-18}\cmidrule(lr){19-20}\cmidrule(lr){21-22}
Dataset & $H$ & MSE & MAE & MSE & MAE & MSE & MAE & MSE & MAE & MSE & MAE & MSE & MAE & MSE & MAE & MSE & MAE & MSE & MAE & MSE & MAE \\
\midrule
\multirow{5}{*}{ETTh1} & 96 & 0.382 & \underline{\textbf{0.393}} & \underline{0.373} & 0.397 & 0.506 & 0.482 & 0.386 & 0.400 & 0.414 & 0.419 & \textbf{0.371} & \textbf{0.393} & 0.375 & 0.395 & 0.386 & 0.405 & 0.414 & 0.419 & 0.384 & 0.402 \\
 & 192 & 0.426 & \underline{\textbf{0.418}} & \underline{\textbf{0.424}} & 0.423 & 0.569 & 0.507 & 0.437 & 0.432 & 0.555 & 0.490 & 0.428 & 0.426 & 0.436 & 0.428 & 0.441 & 0.436 & 0.460 & 0.445 & 0.436 & 0.429 \\
 & 336 & \underline{\textbf{0.466}} & \underline{\textbf{0.438}} & 0.494 & 0.464 & 0.613 & 0.523 & 0.481 & 0.459 & 0.722 & 0.579 & 0.476 & 0.446 & 0.496 & 0.455 & 0.487 & 0.458 & 0.501 & 0.466 & 0.491 & 0.469 \\
 & 720 & \underline{\textbf{0.477}} & \underline{\textbf{0.468}} & 0.532 & 0.501 & 0.636 & 0.545 & 0.519 & 0.516 & 0.713 & 0.596 & 0.487 & 0.470 & 0.520 & 0.484 & 0.503 & 0.491 & 0.500 & 0.488 & 0.521 & 0.500 \\
\cmidrule(lr){2-22}
 & Avg & \underline{\textbf{0.438}} & \underline{\textbf{0.429}} & 0.456 & 0.446 & 0.581 & 0.514 & 0.456 & 0.452 & 0.601 & 0.521 & 0.441 & 0.434 & 0.457 & 0.441 & 0.454 & 0.448 & 0.469 & 0.455 & 0.458 & 0.450 \\
\midrule
\multirow{5}{*}{ETTh2} & 96 & \underline{\textbf{0.291}} & \underline{\textbf{0.340}} & 0.295 & 0.347 & 0.386 & 0.402 & 0.333 & 0.387 & 0.306 & 0.347 & 0.295 & 0.343 & 0.298 & 0.344 & 0.297 & 0.349 & 0.302 & 0.348 & 0.340 & 0.374 \\
 & 192 & \underline{0.372} & \underline{\textbf{0.391}} & 0.397 & 0.413 & 0.482 & 0.454 & 0.477 & 0.476 & 0.397 & 0.408 & \textbf{0.367} & 0.393 & 0.372 & 0.396 & 0.380 & 0.400 & 0.388 & 0.400 & 0.402 & 0.414 \\
 & 336 & \underline{0.431} & \underline{0.432} & 0.439 & 0.448 & 0.469 & 0.459 & 0.594 & 0.541 & 0.455 & 0.454 & \textbf{0.417} & \textbf{0.427} & 0.431 & 0.439 & 0.428 & 0.432 & 0.426 & 0.433 & 0.452 & 0.452 \\
 & 720 & \underline{0.430} & \underline{\textbf{0.443}} & 0.463 & 0.473 & 0.454 & 0.462 & 0.831 & 0.657 & 0.451 & 0.461 & 0.433 & 0.446 & 0.450 & 0.458 & \textbf{0.427} & 0.445 & 0.431 & 0.446 & 0.462 & 0.468 \\
\cmidrule(lr){2-22}
 & Avg & \underline{0.381} & \underline{\textbf{0.402}} & 0.399 & 0.420 & 0.448 & 0.444 & 0.559 & 0.515 & 0.402 & 0.418 & \textbf{0.378} & \textbf{0.402} & 0.388 & 0.409 & 0.383 & 0.407 & 0.387 & 0.407 & 0.414 & 0.427 \\
\midrule
\multirow{5}{*}{ETTm1} & 96 & 0.326 & 0.364 & \underline{0.322} & \underline{0.361} & 0.421 & 0.436 & 0.345 & 0.372 & 0.336 & \underline{0.361} & \textbf{0.311} & \textbf{0.353} & 0.319 & 0.360 & 0.334 & 0.368 & 0.329 & 0.367 & 0.338 & 0.375 \\
 & 192 & \underline{0.364} & 0.388 & 0.369 & 0.388 & 0.480 & 0.465 & 0.380 & 0.389 & 0.383 & \underline{0.387} & \textbf{0.356} & \textbf{0.378} & 0.360 & 0.381 & 0.377 & 0.391 & 0.367 & 0.385 & 0.374 & 0.387 \\
 & 336 & 0.395 & 0.411 & \underline{0.390} & \underline{0.407} & 0.482 & 0.464 & 0.413 & 0.413 & 0.724 & 0.562 & 0.390 & \textbf{0.401} & \textbf{0.389} & 0.403 & 0.426 & 0.420 & 0.399 & 0.410 & 0.410 & 0.411 \\
 & 720 & \underline{0.448} & 0.447 & 0.454 & \underline{0.440} & 0.566 & 0.512 & 0.474 & 0.453 & 0.748 & 0.580 & 0.452 & 0.440 & \textbf{0.447} & 0.441 & 0.491 & 0.459 & 0.454 & \textbf{0.439} & 0.478 & 0.450 \\
\cmidrule(lr){2-22}
 & Avg & \underline{0.383} & 0.403 & 0.384 & \underline{0.399} & 0.487 & 0.469 & 0.403 & 0.407 & 0.548 & 0.473 & \textbf{0.377} & \textbf{0.393} & 0.379 & 0.396 & 0.407 & 0.410 & 0.387 & 0.400 & 0.400 & 0.406 \\
\midrule
\multirow{5}{*}{ETTm2} & 96 & \underline{0.167} & \underline{0.248} & 0.179 & 0.266 & 0.197 & 0.276 & 0.193 & 0.292 & 0.181 & 0.256 & 0.173 & 0.256 & \textbf{0.163} & \textbf{0.246} & 0.180 & 0.264 & 0.175 & 0.259 & 0.187 & 0.267 \\
 & 192 & \underline{0.232} & \underline{0.291} & 0.250 & 0.314 & 0.273 & 0.325 & 0.284 & 0.362 & 0.254 & 0.305 & 0.238 & 0.298 & \textbf{0.229} & \textbf{0.290} & 0.250 & 0.309 & 0.241 & 0.302 & 0.249 & 0.309 \\
 & 336 & \underline{0.291} & \underline{0.331} & 0.331 & 0.373 & 0.326 & 0.356 & 0.369 & 0.427 & 0.321 & 0.349 & 0.301 & 0.340 & \textbf{0.284} & \textbf{0.327} & 0.311 & 0.348 & 0.305 & 0.343 & 0.321 & 0.351 \\
 & 720 & \underline{0.398} & \underline{0.394} & 0.427 & 0.423 & 0.447 & 0.424 & 0.554 & 0.522 & 0.854 & 0.624 & 0.397 & 0.396 & \textbf{0.389} & \textbf{0.391} & 0.412 & 0.407 & 0.402 & 0.400 & 0.408 & 0.403 \\
\cmidrule(lr){2-22}
 & Avg & \underline{0.272} & \underline{0.316} & 0.297 & 0.344 & 0.311 & 0.345 & 0.350 & 0.401 & 0.403 & 0.384 & 0.277 & 0.323 & \textbf{0.266} & \textbf{0.314} & 0.288 & 0.332 & 0.281 & 0.326 & 0.291 & 0.333 \\
\midrule
\multirow{5}{*}{Weather} & 96 & \underline{0.159} & \underline{0.204} & 0.170 & 0.212 & 0.181 & 0.230 & 0.196 & 0.255 & --- & --- & \textbf{0.157} & \textbf{0.200} & 0.158 & 0.203 & 0.174 & 0.214 & 0.177 & 0.210 & 0.172 & 0.220 \\
 & 192 & \underline{0.208} & \underline{0.247} & 0.217 & 0.253 & 0.242 & 0.280 & 0.237 & 0.296 & --- & --- & \textbf{0.206} & \textbf{0.245} & 0.207 & 0.247 & 0.221 & 0.254 & 0.225 & 0.250 & 0.219 & 0.261 \\
 & 336 & \underline{0.264} & \underline{0.289} & 0.273 & 0.293 & 0.293 & 0.315 & 0.283 & 0.335 & --- & --- & \textbf{0.262} & \textbf{0.287} & \textbf{0.262} & 0.289 & 0.278 & 0.296 & 0.278 & 0.290 & 0.280 & 0.306 \\
 & 720 & \underline{\textbf{0.344}} & \underline{0.344} & 0.351 & 0.345 & 0.382 & 0.367 & 0.345 & 0.381 & --- & --- & \textbf{0.344} & 0.342 & \textbf{0.344} & 0.344 & 0.358 & 0.349 & 0.354 & \textbf{0.340} & 0.365 & 0.359 \\
\cmidrule(lr){2-22}
 & Avg & \underline{0.244} & \underline{0.271} & 0.253 & 0.276 & 0.275 & 0.298 & 0.265 & 0.317 & --- & --- & \textbf{0.242} & \textbf{0.269} & 0.243 & 0.271 & 0.258 & 0.278 & 0.259 & 0.273 & 0.259 & 0.287 \\
\midrule
\multicolumn{2}{c}{Overall} & \underline{0.344} & \underline{\textbf{0.364}} & 0.358 & 0.377 & 0.420 & 0.414 & 0.407 & 0.418 & --- & --- & \textbf{0.343} & \textbf{0.364} & 0.346 & 0.366 & 0.358 & 0.375 & 0.356 & 0.372 & 0.364 & 0.380 \\
\bottomrule
\end{tabular}}
\end{table*}

\begin{table*}[h]
\centering
\caption{Experiments on large-channel datasets: \ours{} vs.\ TQNet.  \ours{} MSE is the mean value of 3 seeds with std $\le0.008$. TQNet is the published value. $\Delta$ is the relative gap. A positive value means \ours{} has the higher (worse) MSE.}
\label{tab:bigc}
\small
\setlength{\tabcolsep}{3.5pt}
\begin{tabular}{l ccc ccc ccc}
\toprule
& \multicolumn{3}{c}{Electricity ($C{=}321$)} & \multicolumn{3}{c}{Solar ($C{=}137$)} & \multicolumn{3}{c}{Traffic ($C{=}862$)} \\
\cmidrule(lr){2-4}\cmidrule(lr){5-7}\cmidrule(lr){8-10}
$H$ & \ours{} & TQNet & $\Delta$ & \ours{} & TQNet & $\Delta$ & \ours{} & TQNet & $\Delta$ \\
\midrule
96  & 0.136 & 0.134 & $+$1.5\% & 0.186 & 0.173 & $+$7.5\% & 0.460 & 0.413 & $+$11.4\% \\
192 & 0.153 & 0.154 & $-$0.6\% & 0.200 & 0.199 & $+$0.5\% & 0.467 & 0.432 & $+$8.1\% \\
336 & 0.171 & 0.169 & $+$1.2\% & 0.207 & 0.211 & $-$1.9\% & 0.487 & 0.450 & $+$8.2\% \\
720 & 0.213 & 0.201 & $+$6.0\% & 0.221 & 0.209 & $+$5.7\% & 0.527 & 0.486 & $+$8.4\% \\
\midrule
Avg & 0.168 & 0.164 & $+$2.4\% & 0.204 & 0.198 & $+$3.0\% & 0.485 & 0.445 & $+$9.0\% \\
\bottomrule
\end{tabular}
\end{table*}

\section{Per-horizon accuracy results}
\label{app:perhorizon}

Table~\ref{tab:perhorizon} expands Table~\ref{tab:accuracy} of the main text to full per-horizon MSE and MAE on the five anchor datasets, with the same interpretable/black-box grouping. \ours{} cells are means over 3 seeds, from 2024 to 2026, and the seed standard deviation is not greater than 0.005 MSE on 17 of 20 cells, with a maximum of 0.010 at ETTh2, $H{=}720$. PatchDecomp, TFT, and ProtoTS cells are our runs at the protocol seed and the DLinear and black-box columns are from the TQNet paper's consolidated table.

\section{Breadth on large-channel datasets}
\label{app:bigc}
We also evaluate \ours{} on three large channel datasets: Electricity, Solar, and Traffic. These experiments follow the same protocol as the main results. Table~\ref{tab:bigc} reports results over 3 seeds and 4 horizons, with seed standard deviation no larger than $0.008$, against TQNet's published cells. On Electricity and Solar, \ours{} is essentially at parity: the average gap is $+2.4\%$ and $+3.0\%$, the per-horizon gap ranges from $-0.6\%$ to $+6.0\%$ and from $-1.9\%$ to $+7.5\%$, and \ours{} wins outright at Electricity $H{=}192$ and Solar $H{=}336$. The one substantive gap is Traffic ($C{=}862$), at $+8.1$ to $+11.4\%$ per horizon and $+9.0\%$ on average, shrinking from $+11.4\%$ at $H{=}96$ to about $+8\%$ at longer horizons. A capacity study shows that this gap is not caused by the gate network size, since doubling its width or depth only changes results in the third decimal place. These results provide breadth and scalability evidence. In particular, Traffic trains in only $2.6$GB with channel-independent gate attention, showing that the explanation machinery can scale to hundreds of channels.

\begin{figure*}[t]
\centering
\includegraphics[width=\textwidth]{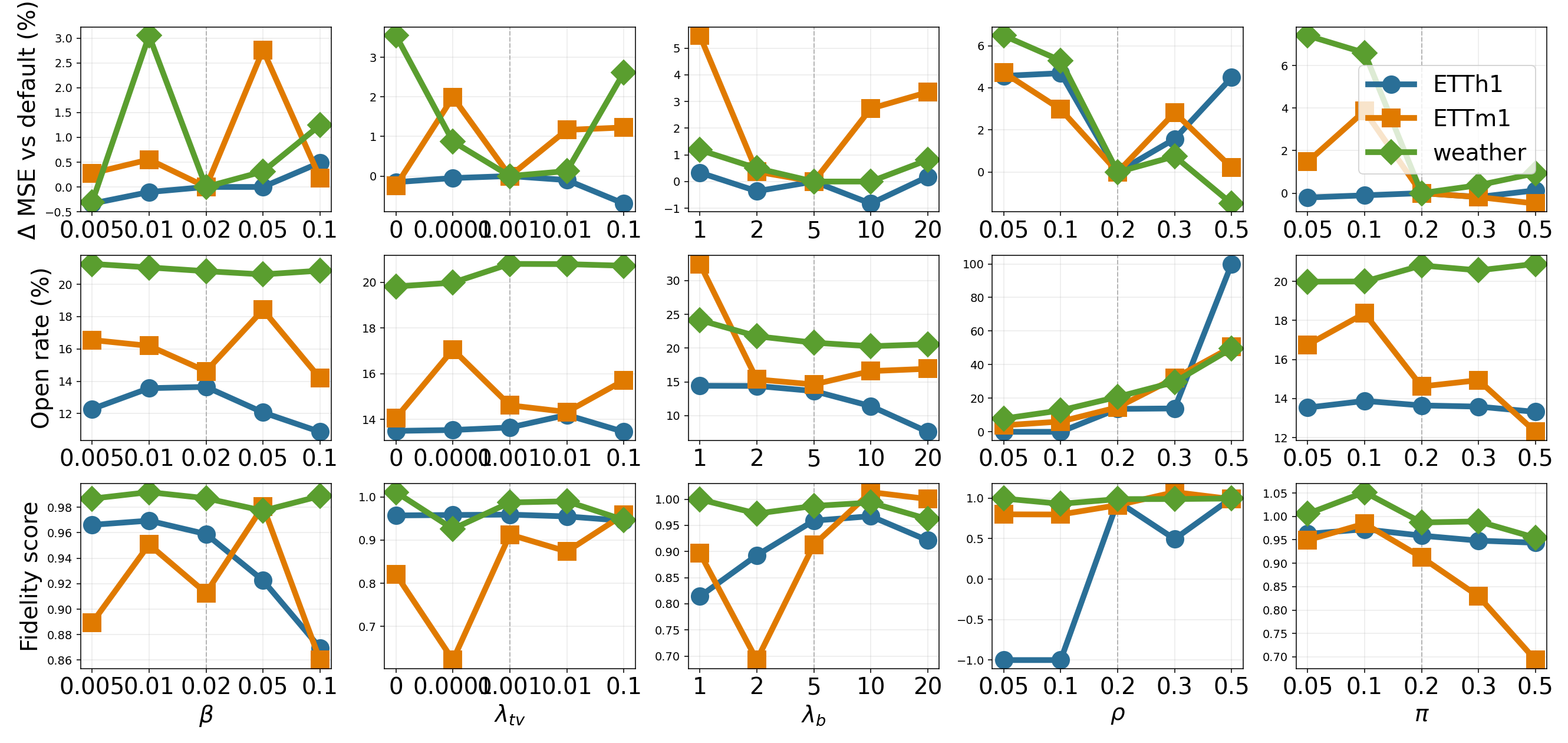}
\caption{One-at-a-time hyperparameter sensitivity on ETTh1, ETTm1 and Weather ($H{=}96$, seed 2024). Columns are the five regularization hyperparameters; the dashed line marks the shared default. Rows: change in test MSE relative to the default (\%), realized base-mask open rate (\%), and matched-budget fidelity score. Accuracy (top) is flat except under $\rho$, the budget dial; the hyperparameters instead move the explanation's sparsity (middle) and faithfulness (bottom), both stable across a wide interior.}
\label{fig:hp}
\end{figure*}

\section{Hyperparameter sensitivity}
\label{app:hp}

The deployed model has five regularization hyperparameters: the information-bottleneck weight $\beta$, the gate temporal-variation weight $\lambda_{tv}$, the budget-penalty weight $\lambda_b$, the open-rate target $\rho$, and the per-gate prior $\pi$. Figure~\ref{fig:hp} sweeps each one at a time about the shared default ($\beta{=}0.02$, $\lambda_{tv}{=}10^{-3}$, $\lambda_b{=}5$, $\rho{=}0.2$, $\pi{=}0.2$) on ETTh1, ETTm1 and Weather at $H{=}96$, reporting the change in test MSE relative to the default, the realized base-mask open rate, and the matched-budget fidelity score.

Two patterns stand out. First, \textbf{accuracy is essentially insensitive to the regularizers}: across every sweep of $\beta$, $\lambda_{tv}$, $\lambda_b$ and $\pi$ the MSE stays within about $3\%$ of the default on all three datasets (top row), so none of these weights needs per-dataset tuning. The only hyperparameter that moves accuracy materially is $\rho$, and it does so because it \emph{is} the budget---the Pareto dial of Section~\ref{sec:rq4}, not a nuisance parameter. Second, \textbf{the regularizers act on the explanation rather than the forecast}: $\rho$ sets the open rate almost linearly (middle row), while $\beta$, $\lambda_b$ and $\pi$ shift it more gently, and the fidelity score (bottom row) stays high, typically above $0.9$. It falls only at the extremes, where a too-large $\beta$ or $\pi$ over-closes the mask on the ETT datasets or a too-small $\rho$ collapses it fully closed (the $\rho\leq0.1$ ETTh1 failure noted in Section~\ref{sec:rq4}), plus a few isolated dips on ETTm1, the noisiest of the three. The default therefore sits in a broad interior region where accuracy, sparsity and faithfulness are simultaneously stable, which is why a single configuration transfers across datasets in Tables~\ref{tab:accuracy} and~\ref{tab:fidelity}.

\begin{table*}[t]
\centering
\caption{ Ground-truth support recovery on the framed benchmark (AUROC, AUR, and AUP). All cells are mean $\pm$ seed standard deviation over 3 seeds.  Dynamask and ExtremalMask receive an oracle sparsity budget. The best performance is marked in bold. $^\dagger$TIMING runs its release defaults, whose segments saturate a 96-step window and score near chance.}
\label{tab:gtmetrics}
\setlength{\tabcolsep}{2pt}
\resizebox{\textwidth}{!}{
\begin{tabular}{l ccc cccccc c}
\toprule
Mode & \textbf{\ours{}} & PatchDecomp& TFT & Occlusion & Saliency & IG & TIMING$^\dagger$ & Dynamask & ExtremalMask & Random \\
\midrule
\multicolumn{11}{c}{\textbf{AUROC}}\\
\texttt{pulse} & \textbf{0.912$\pm$.003} & 0.900$\pm$.002 & 0.549$\pm$.059 & 0.891$\pm$.002 & 0.846$\pm$.006 & 0.839$\pm$.003 & 0.524$\pm$.000 & 0.814$\pm$.001 & 0.868$\pm$.008 & 0.492$\pm$.000 \\
\texttt{decoy} & \textbf{0.878$\pm$.005} & 0.815$\pm$.008 & 0.493$\pm$.027 & 0.827$\pm$.005 & 0.769$\pm$.008 & 0.775$\pm$.008 & 0.530$\pm$.000 & 0.760$\pm$.015 & 0.796$\pm$.003 & 0.502$\pm$.000 \\
\texttt{trend} & \textbf{0.689$\pm$.001} & 0.554$\pm$.007 & 0.515$\pm$.011 & 0.650$\pm$.004 & 0.602$\pm$.007 & 0.637$\pm$.005 & 0.424$\pm$.000 & 0.559$\pm$.004 & 0.566$\pm$.007 & 0.501$\pm$.000 \\
\midrule
\multicolumn{11}{c}{\textbf{AUR}}\\
\texttt{pulse} & \textbf{0.791$\pm$.003} & 0.618$\pm$.007 & 0.180$\pm$.035 & 0.568$\pm$.002 & 0.272$\pm$.022 & 0.230$\pm$.011 & 0.022$\pm$.000 & 0.366$\pm$.007 & 0.343$\pm$.005 & 0.492$\pm$.000 \\
\texttt{decoy} & \textbf{0.755$\pm$.029} & 0.565$\pm$.006 & 0.190$\pm$.068 & 0.416$\pm$.005 & 0.153$\pm$.037 & 0.152$\pm$.020 & 0.021$\pm$.001 & 0.267$\pm$.002 & 0.250$\pm$.008 & 0.502$\pm$.000 \\
\texttt{trend} & \textbf{0.375$\pm$.015} & 0.261$\pm$.007 & 0.150$\pm$.001 & 0.181$\pm$.005 & 0.068$\pm$.001 & 0.068$\pm$.001 & 0.012$\pm$.000 & 0.169$\pm$.001 & 0.159$\pm$.001 & 0.501$\pm$.000 \\
\midrule
\multicolumn{11}{c}{\textbf{AUP}}\\
\texttt{pulse} & 0.471$\pm$.070 & \textbf{0.755$\pm$.014} & 0.193$\pm$.113 & 0.711$\pm$.004 & 0.409$\pm$.068 & 0.611$\pm$.062 & 0.061$\pm$.004 & 0.310$\pm$.010 & 0.294$\pm$.007 & 0.031$\pm$.000 \\
\texttt{decoy} & 0.193$\pm$.050 & \textbf{0.638$\pm$.055} & 0.068$\pm$.037 & 0.375$\pm$.009 & 0.127$\pm$.046 & 0.217$\pm$.058 & 0.065$\pm$.008 & 0.204$\pm$.003 & 0.194$\pm$.012 & 0.030$\pm$.000 \\
\texttt{trend} & 0.343$\pm$.003 & 0.201$\pm$.011 & 0.192$\pm$.007 & \textbf{0.565$\pm$.008} & 0.391$\pm$.007 & 0.520$\pm$.005 & 0.096$\pm$.006 & 0.156$\pm$.001 & 0.152$\pm$.002 & 0.133$\pm$.000 \\
\bottomrule
\end{tabular}}
\end{table*}

\section{Synthetic ground-truth support recovery}
\label{app:synthetic}

In this section, we report the perturbation-free faithfulness evidence that corroborates the real-data result of the matched-budget evaluation (Table~\ref{tab:fidelity}). Real-data fidelity necessarily perturbs inputs and therefore inherits protocol judgment calls; here the generator itself plants the causal support, so explanations are scored against known ground truth with no perturbation operator at all.

\paragraph{Benchmark construction.}
The benchmark instantiates the structural assumptions of Problem~\ref{prob:sif} directly, so that every disclosed component of the model has a legitimate, measurable job. Each sample is a look-back window $X\in\mathbb{R}^{L\times C}$ with a multivariate target $Y\in\mathbb{R}^{H\times C}$, here $L{=}96$, $H{=}24$, $C{=}4$. For each window we draw a per-channel level $\mu_c\sim U(-3,3)$ and scale $s_c\sim U(0.5,1.5)$, a shared cycle phase $\tau\sim U\{0,\dots,W{-}1\}$ with $W{=}24$, and, independently per channel, a driver indicator $z_c\sim\mathrm{Bernoulli}(\tfrac12)$. The window and its future are generated as
\begin{equation}
    \begin{aligned}
    X[t,c] & = \mu_c + s_c\Big(\phi_c\big((t{+}\tau)\bmod W\big) \beta_c(t) \\ 
        & + z_c\, r_c(t) + \varepsilon^{X}_{t,c}\Big),
    \end{aligned}
    \label{eq:synthx}
\end{equation}
\begin{equation}
    \begin{aligned}
    Y[t,c] & = \mu_c + s_c\Big(\phi_c\big((t{+}L{+}\tau)\bmod W\big) \\
            & + z_c\, g_c(t) + \varepsilon^{Y}_{t,c}\Big),
    \end{aligned}
    \label{eq:synthy}
\end{equation}
\begin{align}
M[t,c] &= z_c\cdot \mathds{1}\!\left[t\in\mathrm{supp}(r_c)\right], \label{eq:synthm}
\end{align}
where every injected shape is de-meaned over the window, $\sum_{t}\phi_c(\cdot)=\sum_{t}\beta_c(t)=\sum_{t}r_c(t)=0$, and $\varepsilon^{X},\varepsilon^{Y}$ are i.i.d.\ Gaussian noise. Each term plays a role that mirrors one disclosed component of the model:

\begin{itemize}
\item \emph{Frame} $(\mu_c,s_c)$: the per-channel level and scale, exactly the information the normalization frame ($\mu,\sigma$ of Eq.~\eqref{eq:decomp}) discloses. Because every injection is de-meaned, the window mean equals $\mu_c$ regardless of the driver, so the frame carries \emph{no} driver information and frame and mask are orthogonal by construction.
\item \emph{Cycle} $\phi_c$: a fixed per-channel periodic shape of period $W{=}24$, shared across the whole dataset, entering window and future at the same phase $\tau$; this is the ground-truth counterpart of the learned profile $Q$. It mimics the periodicity of the real hourly benchmarks, where the look-back $L{=}96$ spans four days and the horizon $H{=}24$ spans one, and $\tau$ is handed to the model exactly as in the real protocol.
\item \emph{Background} $\beta_c$: a per-window mixture of random-period sinusoids, present in $X$ but absent from $Y$ (note $\beta_c$ appears in Eq.~\eqref{eq:synthx} but not Eq.~\eqref{eq:synthy}), hence causally irrelevant by construction. A faithful explanation must leave it out.
\item \emph{Drivers} $r_c$ (present iff $z_c{=}1$): the sparse in-channel causal event, and its type defines the mode. \texttt{pulse} plants a 6-step bump of amplitude $U(2,5)$ with random sign; \texttt{trend} plants a 24-step ramp with slope $U(0.05,0.25)$ and random sign, whose realized rise the future response $g_c$ carries and decays over the horizon; \texttt{decoy} plants the \texttt{pulse}. It then adds an attenuated, noise-corrupted copy elsewhere in the same channel. The copy is predictive of $Y$, but strictly inferior to the clean pulse. It is \emph{excluded} from the ground truth $M$.
\end{itemize}

The forecast term $g_c$ is the driver's effect on the future (for \texttt{trend}, the ramp's realized rise decaying over the horizon; for \texttt{pulse}/\texttt{decoy}, its analogous response). A channel with $z_c{=}0$ has a pure frame-and-cycle future, so a closed mask there asserts precisely that the channel followed the disclosed pattern, and by Eqs.~\eqref{eq:synthx}--\eqref{eq:synthy} that assertion is true. The ground-truth mask $M$ of Eq.~\eqref{eq:synthm} marks exactly the injected driver positions. 
Splits are 4000/1000/1000 windows; explanations are scored on 256 test windows with AUROC, AUP, and AUR against $M$.

\paragraph{Model and baselines.}
The model under explanation is the deployed architecture of the main text, with the frame, the cycle profile at $W{=}24$, and the base mask as the explanation, trained for 150 epochs at patch length 6, $\beta{=}0.1$, $\pi{=}0.05$ (three seeds; six where noted). 
The post-hoc baselines explain this same model. They are saliency, IG, TIMING, patch occlusion, and a random floor. Two are test-time mask optimizers, Dynamask and ExtremalMask. Both run at the oracle budget, namely the ground-truth mask's own density.
TFT and PatchDecomp are native-versus-native comparisons: each is trained fresh on the benchmark in its strongest configuration (multivariate output with instance normalization), and its own explanation is read off exactly as its paper prescribes. Their architectures take no phase input, which we disclose; four full cycles are visible in every window, so the phase is inferable from the data they see.

\paragraph{Analysis.}
\ours{} ranks first in every AUROC cell in Table~\ref{tab:gtmetrics}. It also performs especially well on recall. Its AUR is between $0.76$ and $0.79$ on \texttt{pulse} and \texttt{decoy}, while every baseline stays at or below $0.62$. This shows that the budgeted mask recovers most of the planted support, whereas point scoring methods only identify fragments. The \texttt{decoy} setting gives a clearer causal test of the information bottleneck. PatchDecomp is close on the clean \texttt{pulse} setting, with $0.900$ versus $0.912$, but drops to $0.815$ on \texttt{decoy} because its unconstrained forecaster uses the decoy. In contrast, our mask rejects the decoy and reaches $0.878$. The effect is controlled by $\beta$. With $\beta{=}0.02$, \texttt{decoy} AUROC is only $0.09$ to $0.36$, which means the model actively selects the decoy. With $\beta{=}0.1$, it rises to $0.85$ to $0.88$. Spurious exclusion is therefore controlled by the objective, not by accident. The cycle component also matters. Removing it drops \texttt{pulse} AUROC from $0.876$ to $0.631$ and increases the mask's false positive mass from $0.68$ to $0.955$, because the budget is wasted on relearning periodic structure. TFT, by contrast, stays near chance in all modes, consistent with the attention is not explanation phenomenon \citep{attention2019}.

The results also show two costs. First, \ours{} trades precision for recall. It does not lead on AUP in Table~\ref{tab:gtmetrics}, because the sufficiency objective tends to select extra tokens around each event instead of isolating only the exact causal tokens. The \texttt{trend} setting is difficult for every method. Our native explanation reaches $0.689$, while the best post hoc method reaches $0.650$. This happens because the target depends on the ramp's realized rise, so the truly sufficient evidence is smaller than the full annotation. As a result, all methods lose recall under this label. Second, the sufficiency constraint introduces an accuracy cost, and this benchmark makes that cost look large. PatchDecomp can read the whole window and reaches roughly half our MSE, for example $0.16$ versus $0.30$ on \texttt{pulse}. Under the real data protocol, however, the same constraint costs only about $3\%$. The synthetic setting is a worst case for a budgeted reader, since all useful signal is concentrated in the few tokens the model must find.

\begin{figure}[t]
\centering
\includegraphics[width=0.8\columnwidth]{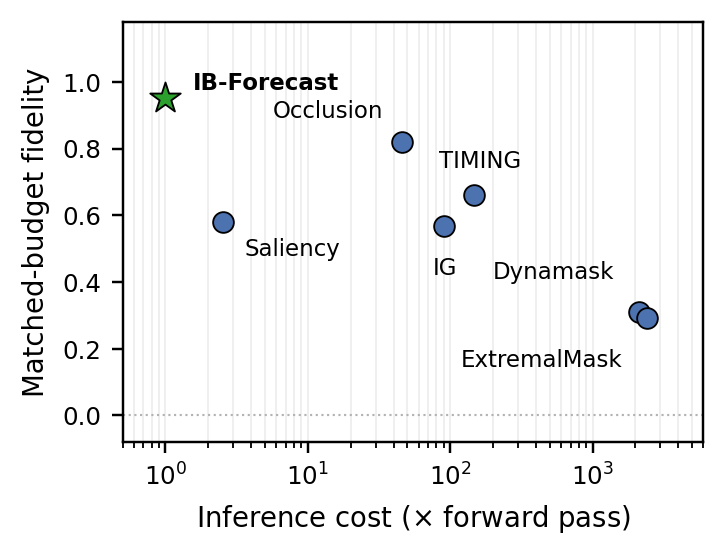}
\caption{Fidelity against inference cost per explanation on the ETTh1 dataset at $H{=}96$, with cost on a logarithmic axis measured in forward-equivalent passes. }
\label{fig:cost}
\end{figure}

\section{Efficiency analysis}
\label{sec:efficiency}
We present the efficiency analysis in terms of the measured inference cost per explanation, timed on a GPU and expressed in forward-equivalent passes. As shown in Figure~\ref{fig:cost}, among the post-hoc methods, saliency costs about $2.5$, occlusion about $46$, IG about $91$, and TIMING about $148$ forward-equivalent passes, and each trails the native mask by $0.13$ to $0.37$ fidelity. TFT and PatchDecomp also explain at near-zero cost but score negative fidelity.
Figure~\ref{fig:cost} compares fidelity against inference cost per explanation. The native explanation is emitted by the forward pass and adds no computation. Among post-hoc methods, occlusion, the strongest on fidelity, requires 56 to 168 forward passes per window; TIMING requires 50 forward and backward passes and trails the native mask by 0.08 to 0.37 fidelity, with IG at 32 such passes and saliency at a single backward pass. The test-time mask optimizers spend 500 optimization steps per batch, each a forward and backward pass, and still trail by more than 0.5. TFT and PatchDecomp also explain at zero cost but score negative fidelity (Table~\ref{tab:fidelity}), so low cost alone does not confer faithfulness. 
Finally, explanation scales to Traffic ($C{=}862$) with no memory overhead beyond the forecaster itself, at 2.6\,GB total, via channel-independent gate attention.

\end{document}